\documentclass[letterpaper,journal]{IEEEtran}
\usepackage{amsmath,amssymb,amsfonts}
\usepackage{amsthm}
\usepackage{algorithmic}
\usepackage{array}
\usepackage[caption=false,font=normalsize,labelfont=sf,textfont=sf]{subfig}
\usepackage{textcomp}
\usepackage{stfloats}
\usepackage{url}
\usepackage{bbm}
\usepackage{cite}
\newtheorem{lemma}{Lemma}
\newtheorem{theorem}{Theorem}
\usepackage{comment}
\usepackage{tikz}
\usetikzlibrary{positioning}
\newtheorem{assumption}{Assumption}
\usepackage{verbatim}
\usepackage{graphicx}
\usepackage{algorithm}
\def\BibTeX{{\rm B\kern-.05em{\sc i\kern-.025em b}\kern-.08em
    T\kern-.1667em\lower.7ex\hbox{E}\kern-.125emX}}
\usepackage{balance}

\begin{document}
\title{Decentralized Multitask Learning over Learned Task Graphs}
\author{Zirui Wan~\IEEEmembership{Student Member,~IEEE}, and Stefan Vlaski~\IEEEmembership{Member,~IEEE}
\thanks{Zirui Wan and Stefan Vlaski are with the Department of
Electrical and Electronic Engineering, Imperial College London,
London SW7 2AZ, U.K. (e-mail: zirui.wan22@imperial.ac.uk;
s.vlaski@imperial.ac.uk).}}
\allowdisplaybreaks

\maketitle

\begin{abstract}
This paper investigates decentralized multitask learning over networks when the underlying task relationships are unknown. While existing graph-regularized multitask frameworks typically assume a known structure, practical settings often require learning inter-task dependencies directly from distributed data. We propose a decentralized two-phase strategy that first estimates a generalized graph Laplacian from noisy non-cooperative stochastic gradient iterates, and subsequently exploits the learned graph to enable cooperative multitask diffusion learning. This framework is motivated by a Gaussian Markov random field prior, which gives rise to a decentralized maximum likelihood estimator for the graph Laplacian. 
The analysis quantifies the Laplacian estimation error and its propagation to the steady-state performance of the multitask diffusion recursion, and introduces a topology sensitivity index to capture the effect of network heterogeneity.
Simulation results corroborate the theoretical findings and demonstrate that cooperation enabled by the learned task graph significantly improves performance over non-cooperative learning, while approaching the true-graph baseline when the estimation stepsize is sufficiently small.
\end{abstract}

\begin{IEEEkeywords}
    Decentralized learning, Gaussian graphical models, graph learning,
    graph signal processing, multitask learning.
\end{IEEEkeywords}

\section{Introduction}
\IEEEPARstart{M}{odern} distributed networks generate large volumes of data, motivating federated and decentralized learning frameworks that move storage and computation toward the network edge~\cite{stefan_review}. Early works primarily focused on \emph{consensus} or \emph{single-task} formulations, where all agents are required to agree on a common global model or decision~\cite{Consensus1,Consensus2,Adaptation,EXTRA_consensus,fed_consensus,fed_consensus2}. In many practical applications, however, agents face heterogeneous objectives and non-identical local data distributions, making strict consensus statistically suboptimal~\cite{Multitask1,Multitaks2}. To address such heterogeneity, more recent works have developed implicit or non-parametric approaches, including meta-learning~\cite{meta1,meta2} and proximal personalization methods~\cite{proximal1,proximal2}. While effective for personalization, these approaches capture task relationships only implicitly through shared parameters, and therefore offer limited interpretability of the underlying task dependencies. In many settings, however, tasks exhibit structured pairwise relationships that are unknown \emph{a priori} and must be inferred directly from distributed data. This challenge has motivated the study of \emph{multitask learning over graphs}, where each agent maintains its own model while leveraging inter-task relationships through explicit structural priors.

A principled approach to incorporating prior knowledge about task relationships into multitask learning is through structured regularization. Consider a networked learning scenario in which each agent $k$ seeks to estimate a parameter vector \( w_k^o \in \mathbb{R}^{M}\) by minimizing its individual cost function:
\begin{equation}
    w_k^{o}=\operatorname*{arg\,min}_{w_k}\; J_k(w_k).
\end{equation}
When information about inter-task dependencies is available, it can be embedded into the learning process by augmenting the objective with a regularization term that promotes desirable structure in the solution~\cite{kernels,regularization,stefan_reg1}. This leads to the regularized multitask formulation
\begin{equation}
\operatorname*{arg\,min}_{{\scriptstyle \mathcal{W}}}
\;\mathcal{J}({\scriptstyle \mathcal{W}})
\;+\;
\frac{\eta}{2}\,{\scriptstyle \mathcal{W}}^\top \mathcal{R}\,{\scriptstyle \mathcal{W}},
\label{eq:multitask_regularized}
\end{equation}
where ${\scriptstyle \mathcal{W}}\triangleq\operatorname{col}\{w_{1},\ldots,w_{K}\} \in \mathbb{R}^{MK}$ concatenates the local parameters $w_k$ from $K$ agents, 
$\mathcal{J}({\scriptstyle \mathcal{W}})\triangleq \sum_{k=1}^{K} J_k(w_k)$ denotes the aggregate network cost, $\mathcal{R} \in \mathbb{R}^{MK \times MK}$ is a positive semidefinite matrix that encodes structural relationships among tasks, and $\eta>0$ controls the strength of coupling across agents.

In networked settings, such prior structure is often expressed through \emph{smoothness} of the model parameters over an underlying task graph, whereby related agents are encouraged to learn similar parameter vectors~\cite{nassifSPM}. This notion can be captured by choosing $\mathcal{R}=\mathcal{L}\triangleq L\otimes I_M$, where $L$ is a graph Laplacian that encodes pairwise task similarities. Under this choice, the regularizer $\frac{\eta}{2}{\scriptstyle \mathcal{W}}^\top \mathcal{L}{\scriptstyle \mathcal{W}}$ promotes smooth variation of the parameter vectors across the graph, enabling the algorithm to exploit structural dependencies while preserving local adaptability~\cite{stefan_reg1}. Moreover, this graph-regularized formulation admits a probabilistic interpretation through Gaussian Markov random field (GMRF) priors, where the Laplacian serves as the precision matrix governing conditional dependencies among tasks over graph~\cite{GMRF1,GMRF2,GMRF3}. Such graph-based multitask frameworks have been shown to significantly improve estimation accuracy and tracking performance over independent or consensus-based approaches, particularly in heterogeneous environments where task similarity is localized rather than global~\cite{Multitask1,Multitaks2,stefan_reg1,stefan_reg2, wan}.

\begin{figure}[t!]
    \centering
    \includegraphics[width=1\linewidth]{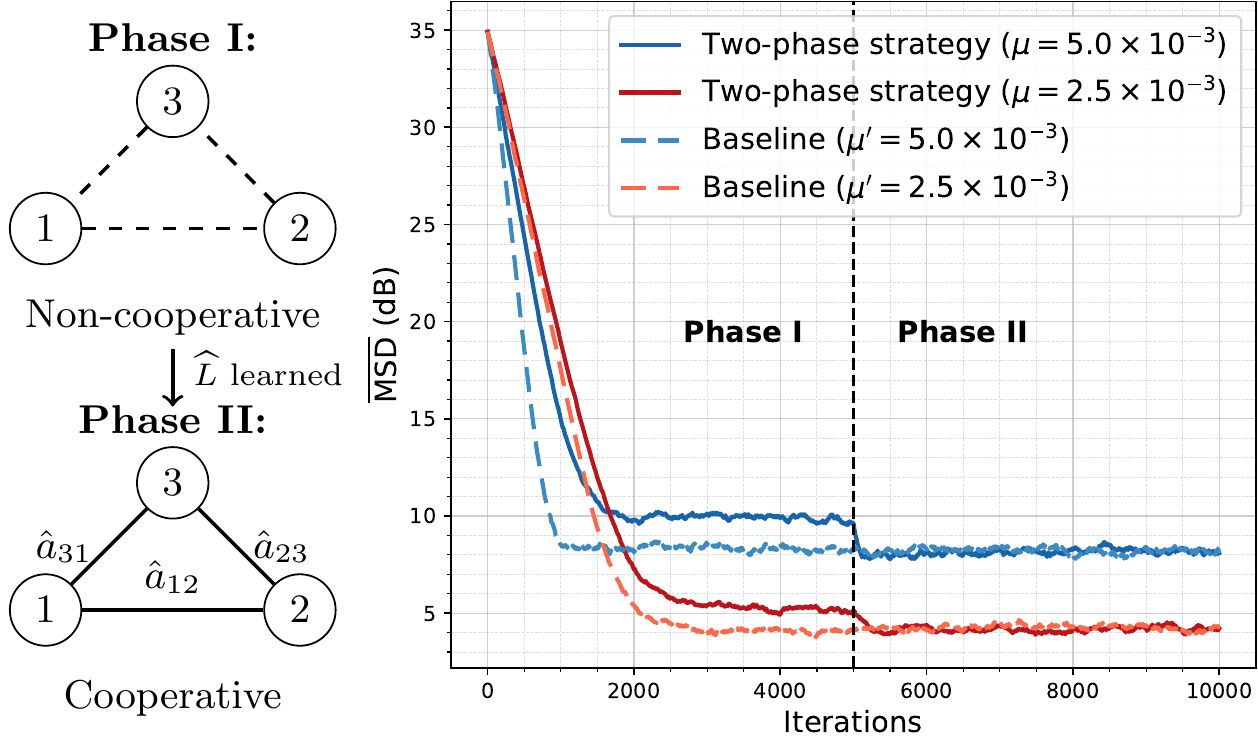}
    \caption{(Left) Illustration of the proposed two-phase strategy: Phase~I, agents operate non-cooperatively to learn the underlying task graph, and Phase~II uses the learned graph for cooperative multitask learning. (Right) Network mean-square deviation $\overline{\mathrm{MSD}}$ versus iteration for two choices of the stepsize. After the graph is learned in Phase~I, the Phase~II performance approaches that of the baseline with known task relationships (dashed curves) under the same stepsizes.}
    \vspace{-0.8em}
    \label{fig:illustration}
\end{figure}
However, a fundamental limitation of most existing graph-based multitask learning approaches is the assumption that the underlying task-relationship graph---including both its connectivity pattern and edge weights---is fully known a priori. In practice, this assumption is often unrealistic, since the true dependencies among tasks are typically implicit, data-driven, and may evolve over time. Consequently, specifying an accurate graph structure in advance is difficult, and the task-relationship graph must instead be inferred directly from the available data, a problem commonly referred to as \emph{graph learning}.

Moreover, even when partial structural knowledge is available, existing graph learning methods typically rely on centralized processing, requiring global covariance estimation or large-scale matrix inversion~\cite{how_to_learn_graph,GMRF3,graph_learn2,wan}. Such requirements are incompatible with decentralized settings. These challenges motivate \emph{decentralized} learning strategies that can jointly infer task relationships and perform multitask learning directly from distributed data, while ensuring statistical accuracy of the learned graph and characterizing how graph estimation errors affect the performance of the resulting multitask learning algorithm.

In this paper, we propose a fully decentralized strategy that jointly learns local models and their inter-task relationships, assuming that only the local connectivity is known. The dependencies among tasks are modeled through a GMRF, whose \emph{unknown} precision matrix is constrained to lie in the space of valid \emph{generalized graph Laplacians} (GGLs)~\cite{graph_learn2,graph_learn3}. To estimate the graph structure, we employ a decentralized Laplacian learning technique based on the noisy non-cooperative estimates of the local models~\cite{DML3}. The resulting Laplacian estimate is then incorporated into the multitask learning algorithm to promote structured cooperation among agents. The overall procedure of the proposed strategy is illustrated in Fig.~\ref{fig:illustration} (left).

Simulation results demonstrate that, despite being learned from noisy data, the estimated Laplacian enables improved performance compared to the non-cooperative baseline, as shown in Fig.~\ref{fig:illustration} (right). From a theoretical perspective, we derive explicit bounds on the Laplacian estimation error in the small-stepsize and high-dimensional regimes, revealing an $O(\mu)$ dependence on the non-cooperative learning stepsize $\mu$ and quantifying the influence of graph topology on estimation quality. We further analyze how errors in the learned graph propagate and affect the performance of the resulting decentralized multitask learning algorithm. 

\subsection{Relation to Existing Studies}

\noindent\textbf{Multitask learning over graphs:}
Prior works such as~\cite{Multitask1,network_lasso,Multitaks2,stefan_sub} formulate learning over multitask graphs, where multiple task-specific parameter vectors are learned jointly while exploiting structural relations among tasks. The use of graph Laplacian regularization to encode inter-task similarities was introduced in~\cite{kernels,regularization,stefan_reg1}, where adaptive learning strategies were developed for streaming data settings. In~\cite{distributed_smoothness}, task parameters are modeled under a smoothness condition induced by a GMRF prior, leading to a maximum a posteriori (MAP) formulation. Building on this framework,~\cite{stefan_reg1,stefan_reg2} analyze iterative solutions and provide detailed performance characterizations, showing how network topology, stepsize, and regularization strength influence the steady-state behavior of multitask learning algorithms.

Another line of work adopts implicit, non-parametric strategies such as distributed meta-learning~\cite{meta1,meta2}, which cooperatively learn shared meta-parameters that can be rapidly adapted to heterogeneous tasks. Proximal personalization methods take a related but distinct approach by allowing each agent to learn a task-specific model while regularizing its deviation from a shared reference model. For example,~\cite{proximal1} learns personalized local models by regularizing them toward a shared global reference model. In~\cite{proximal2}, client drift is mitigated by encouraging each local model to remain close to the global reference while moving away from its previous local model.

A common limitation of the aforementioned approaches is that task relationships are either assumed known a priori, as in many graph-based multitask methods, or are encoded only implicitly through shared parameters or adaptation mechanisms, as in meta-learning and proximal personalization approaches. In contrast, our work does not require prior knowledge of the graph weights; instead, we assume only local connectivity among agents and explicitly infer the task relationships from distributed data.

\textbf{Graph learning:}
We consider the problem of learning the unknown edge weights of a graph Laplacian from data. Early works such as~\cite{how_to_learn_graph,GMRF3} estimate the graph Laplacian in a centralized manner from graph signals generated under a GMRF prior. The work in~\cite{graph_learn2} studies the estimation of several classes of graph Laplacians, including GGLs, diagonally dominant generalized graph Laplacians (DDGLs), and combinatorial graph Laplacians (CGLs), under structural constraints. More recently,~\cite{DML1,DML2,DML3,graph_learn3} investigate decentralized graph Laplacian estimation based on marginal likelihood formulations derived from the Schur complement of the marginal covariance~\cite{schur}, enabling graph recovery using only local information.
In this work, we adopt the GGL model for the graph Laplacian, whose positive definiteness guarantees matrix invertibility and enables the use of the Schur complement identity. GGLs have also been shown effective for modeling structured signals such as images and videos~\cite{GGL1,GGL2}. Our method builds on maximizing the marginal likelihood (MML) approach proposed in~\cite{graph_learn3}. The key distinction of our work is that the graph is estimated from noisy graph signal observations arising from stochastic learning dynamics, and we provide an analysis of how this noise results in Laplacian estimation errors.

\textbf{Joint learning of models and graphs:}  
The framework in~\cite{joint_fed} operates in a federated setting with centralized coordination, where a server aggregates updates and optimizes shared variables, whereas our approach is fully decentralized and relies solely on local neighbor interactions. Earlier related work in~\cite{Clustering} studies distributed clustering and learning over networks, where agents adaptively infer which neighbors share similar objectives and should cooperate; however, it focuses on clustering formation rather than learning an explicit weighted task-relationship graph. Works such as~\cite{joint2,joint3} also learn graph structures alongside
personalized models in a decentralized manner, but primarily aim to
construct a useful collaboration graph.
Such a graph need not faithfully recover the underlying statistical task
dependencies, leaving open how accurately these dependencies can be
estimated and how estimation errors affect subsequent learning. In contrast,
we treat the task graph as an underlying statistical quantity and study its recovery from noisy stochastic
iterates. This formulation enables finite-sample estimation guarantees and helps characterize
the propagation of graph error to downstream multitask learning. The most closely related work is our prior study~\cite{wan}, however, that method is not decentralized and does not provide convergence or performance guarantees for the resulting multitask learning algorithm.

\subsection{Summary of Contributions}

\noindent The main contributions of this work are summarized as follows:

\begin{itemize}
\item \textbf{Decentralized joint graph and multitask learning:}
We propose a two-phase decentralized framework that jointly learns task relationships and performs graph-regularized multitask adaptation using only local connectivity information.

\item \textbf{Laplacian estimation from noisy learning dynamics:}
We develop a decentralized Laplacian learning method based on noisy non-cooperative stochastic iterates and establish finite-sample error bounds for the estimated graph.

\item \textbf{Error propagation analysis:}
We quantify how graph estimation errors affect the steady-state performance of the downstream decentralized multitask learning algorithm.
\end{itemize}
Taken together, these provide an end-to-end framework for decentralized multitask learning with performance guarantees despite the lack of prior knowledge of task relationships.
\subsection{Notation}
\noindent Throughout the paper, all vectors are column vectors. Random quantities are in boldface; matrices are uppercase, and vectors/scalars are lowercase. $\otimes$ denotes the Kronecker product, $\operatorname{diag}(\cdot)$ constructs a block-diagonal matrix, and $\operatorname{col}\{\cdot\}$ stacks the columns of a matrix into a vector. The notation $\|\cdot\|$ refers to the spectral norm for matrices and the $\ell_2$-norm for vectors.

\section{Models and Proposed Algorithm}
\subsection{GMRF Prior for Model Parameters}
\noindent Our graph-learning formulation is motivated by modeling the true model parameters as realizations of a zero-mean $M$-variate Gaussian Markov random field (GMRF)~\cite{GMRF1} over a connected, undirected graph $\mathcal{G}=(\mathcal{V},\mathcal{E})$. Specifically, each agent $k$ from the vertex set $\mathcal{V}$ has a parameter $w_k^o \in \mathbb{R}^M$ treated as a Gaussian random vector, and the relationships among agents are encoded by the edges of a connected, undirected graph. Each edge $(k,\ell)\in \mathcal{E}$ is assigned a nonnegative weight $a_{k\ell}$ that reflects the similarity between agents $k$ and $\ell$. The neighborhood of node $k$ is defined as  $\mathcal{N}_k \triangleq \{\, \ell \mid (k,\ell) \in \mathcal{E} \,\}\cup\{k\}$.

In this work, the graph Laplacian matrix $L$ is restricted to the class of Generalized Graph Laplacians (GGLs)~\cite{graph_learn2}, defined as
\begin{equation}
    L \triangleq D - A + V,
    \label{eq:Laplacian}
\end{equation}
where $A \in \mathbb{R}^{K \times K}$ is a symmetric weighted adjacency matrix with entries $A_{k\ell}=a_{k\ell}$, $D=\mathrm{diag}(d_1,\ldots,d_K)$ is the weighted degree matrix with $d_k=\sum_{\ell \in \mathcal{N}_k} a_{k\ell}$, and $V=\mathrm{diag}(v_1,\ldots,v_K)$ is a diagonal self-loop matrix with nonnegative $v_k$, and at least one $v_k>0$. By construction, $L$ is symmetric and positive definite.

\begin{assumption}[GMRF model]
\label{assumption:GMRF}
The true parameter vector ${\scriptstyle \mathcal{W}^o} \triangleq \mathrm{col}\{w_1^o,\ldots,w_K^o\}$ is sampled from a multivariate Gaussian distribution:
\begin{equation}
     \boldsymbol{{\scriptstyle \mathcal{W}}}^o \sim \mathcal{N}\{0,\,\mathit{\Sigma}\},
     \label{eq:GMRF}
\end{equation}
where $\mathit{\Sigma} = \Sigma\otimes I_M \triangleq L^{-1} \otimes I_M$. The Kronecker form means that the same graph-induced dependency encoded by $L^{-1}$ is repeated across all $M$ coordinates of the parameter vectors. In other words, the inter-agent dependency structure is shared across dimensions.
\end{assumption}
\noindent The probability density function is given by
\begin{equation}
f_{{\scriptstyle \mathcal{W}^o}}({\scriptstyle \mathcal{W}^o})
=
\frac{
\exp\!\left( -\tfrac{1}{2} ({\scriptstyle \mathcal{W}^o})^\top \mathcal{L}\, {\scriptstyle \mathcal{W}^o} \right)
}{
\sqrt{\det(2\pi \mathcal{L}^{-1})}
},
\label{eq:probability_density}
\end{equation}
where $\det(\cdot)$ denotes the determinant operator.
The quadratic form in \eqref{eq:probability_density} admits the explicit decomposition
\begin{equation}
    ({\scriptstyle \mathcal{W}^o})^\top \mathcal{L}\, {\scriptstyle \mathcal{W}^o}
=
\frac{1}{2}\sum_{k=1}^{K}\sum_{\ell \in \mathcal{N}_k}
a_{k\ell}\,\lVert w_k^o - w_\ell^o \rVert^2
\;+\;
\sum_{k=1}^{K} v_k \lVert w_k^o \rVert^2.
\end{equation}
This representation reveals that the GMRF prior promotes smoothness of the parameter vectors over the underlying graph, while the diagonal self-loop terms $v_k$ control the magnitude of each local model.

\subsection{Decentralized Multitask Learning}
\noindent Under the GMRF prior~\eqref{eq:GMRF}, the MAP estimator of ${\scriptstyle \mathcal{W}}^o$
leads to a graph-regularized optimization problem. Let
$\boldsymbol{{\scriptstyle \mathcal{X}}}_j
\triangleq
\operatorname{col}\{\boldsymbol{x}_{1,j},\ldots,
\boldsymbol{x}_{K,j}\}$
collect the observations received across the network at time~$j$.
For each agent, define the instantaneous negative log-likelihood
\begin{equation}
Q_k(w_k;\boldsymbol{x}_{k,j})
\triangleq
-\log f_{\boldsymbol{x}_k\mid w_k}
(\boldsymbol{x}_{k,j}\mid w_k).
\label{eq:instantaneous_loss}
\end{equation}
Then, the empirical risk based on the first
$i$ observations is
\begin{equation}
\widehat{\mathcal{J}}_i({\scriptstyle \mathcal{W}})
\triangleq
\frac{1}{i}\sum_{j=1}^{i}\sum_{k=1}^{K}
Q_k(w_k;\boldsymbol{x}_{k,j}).
\label{eq:empirical_risk}
\end{equation}
The MAP estimator is therefore given by
\begin{align}
{\scriptstyle \mathcal{W}}_i^{*}
&=
\operatorname*{arg\,min}_{{\scriptstyle \mathcal{W}}}
\left\{
-\log
f_{\{\boldsymbol{{\scriptstyle \mathcal{X}}}_j\}_{j=1}^{i}
\mid {\scriptstyle \mathcal{W}}}
-\log f_{\scriptstyle \mathcal{W}}
({\scriptstyle \mathcal{W}})
\right\} \notag\\
&=
\operatorname*{arg\,min}_{{\scriptstyle \mathcal{W}}}
\left\{
\widehat{\mathcal{J}}_i({\scriptstyle \mathcal{W}})
+
\frac{1}{2i}
{\scriptstyle \mathcal{W}}^\top
\mathcal{L}
{\scriptstyle \mathcal{W}}
\right\},
\label{eq:empirical_MAP}
\end{align}
where the second equation uses \eqref{eq:empirical_risk} and Assumption~\ref{assumption:GMRF}. Thus, the GMRF prior
induces a Laplacian regularizer in the empirical MAP problem
\eqref{eq:empirical_MAP}.

For identically distributed observations over time, the corresponding
expected risk is
\begin{equation}
\mathcal{J}({\scriptstyle \mathcal{W}})
\triangleq
\mathbb{E}\!\left[
\widehat{\mathcal{J}}_i({\scriptstyle \mathcal{W}})
\right]
=
\sum_{k=1}^{K}
\mathbb{E}_{\boldsymbol{x}_k}
\left[
Q_k(w_k;\boldsymbol{x}_k)
\right]
\triangleq
\sum_{k=1}^{K}J_k(w_k).
\label{eq:cost_function}
\end{equation}
Motivated by the MAP estimator
\eqref{eq:empirical_MAP}, we consider the regularized risk
minimization problem
\begin{equation}
{\scriptstyle \mathcal{W}}^{o}_{L}
=
\operatorname*{arg\,min}_{{\scriptstyle \mathcal{W}}}
\left\{
\mathcal{J}({\scriptstyle \mathcal{W}})
+
\frac{\eta}{2}
{\scriptstyle \mathcal{W}}^\top
\mathcal{L}
{\scriptstyle \mathcal{W}}
\right\},
\label{eq:multitask_cost}
\end{equation}
where $\eta > 0$ controls the strength of the regularization. For the finite-sample MAP problem, averaging the log-likelihood gives $\eta=1/i$. In a commonly used setting of mean-square loss with Gaussian observation noise of variance $\sigma^2$, the corresponding
coefficient is $\eta=\sigma^2/i$. In streaming settings, these
quantities may be unknown or time-varying, so we treat $\eta$ as a
tunable parameter controlling the strength of cooperation.

We consider solving~\eqref{eq:multitask_cost} in a stochastic decentralized setting, where the data distributions are unknown and the exact gradient
$\nabla \mathcal{J}({\scriptstyle \mathcal{W}})$ is unavailable. When the graph Laplacian $\mathcal{L}$ is known, the agents can use stochastic gradient approximations in the following diffusion recursion~\cite{Adaptation,stefan_reg1,stefan_reg2}:
\begin{equation}
\boldsymbol{{\scriptstyle \mathcal{W}}}_{L,i}
=
\left(I_{MK}-\mu^{\prime}\eta\mathcal{L}\right)
\left(
\boldsymbol{{\scriptstyle \mathcal{W}}}_{L,i-1}
-
\mu^{\prime}\,\widehat{\nabla \mathcal{J}}
\!\left(\boldsymbol{{\scriptstyle \mathcal{W}}}_{L,i-1}\right)
\right),
\label{eq:DGD_recursion_kownL}
\end{equation}
where $\mu^\prime>0$ is a constant stepsize and
$\widehat{\nabla \mathcal{J}}(\cdot)$ is a stochastic gradient approximation.
The sparsity of $\mathcal{L}$ ensures that each update requires communication only among neighboring agents.

In practice, however, $\mathcal{L}$ may be unknown and is replaced by an estimate
$\widehat{\mathcal{L}}$, leading to
\begin{equation}
\boldsymbol{{\scriptstyle \mathcal{W}}}_{\widehat{L},i}
=
\left(I_{MK}-\mu^{\prime}\eta\widehat{\mathcal{L}}\right)
\left(
\boldsymbol{{\scriptstyle \mathcal{W}}}_{\widehat{L},i-1}
-
\mu^{\prime}\,\widehat{\nabla \mathcal{J}}
\!\left(\boldsymbol{{\scriptstyle \mathcal{W}}}_{\widehat{L},i-1}\right)
\right),
\label{eq:DGD_recursion}
\end{equation}
which is the recursion studied in this work. The key challenge in deriving performance guarantees for \eqref{eq:DGD_recursion}, relative to \eqref{eq:DGD_recursion_kownL}, is that errors in the estimated graph Laplacian \(\widehat{L}\) propagate through the learning dynamics of $\boldsymbol{{\scriptstyle \mathcal{W}}}_{\widehat{L}}$. Capturing and resolving this coupling constitutes the central analytical contribution of this work.

\subsection{Non-Cooperative Phase for Laplacian Estimation}
\noindent In practice, the graph Laplacian $L$ is unknown and must be inferred from data under the structural prior in~\eqref{eq:GMRF}.
Since the true parameter vector ${\scriptstyle \mathcal{W}}^{o}$, which encodes inter-task relationships, is not directly observable, each agent first constructs a local estimate in a non-cooperative manner, without access to $L$.
Specifically, agent $k$ runs an independent stochastic gradient descent recursion of the form:
\begin{equation}
\boldsymbol{w}_{k,i}
=
\boldsymbol{w}_{k,i-1}
-
\mu\,\widehat{\nabla J}_k(\boldsymbol{w}_{k,i-1}),
\label{eq:GD recursion}
\end{equation}
using only its own data. $\mu>0$ is a constant stepsize. A standard asymptotic result for recursion~\eqref{eq:GD recursion} is stated next and will be used in the analysis of the following section.

\begin{lemma}[Steady-state Error Covariance~\cite{Adaptation,Clustering}]
\label{Lemma: Steady-state Error Covariance}
For a sufficiently small stepsize $\mu$, the recursion in~\eqref{eq:GD recursion} satisfies, at steady-state,
\begin{equation}
    \lim_{\mu\to0}\limsup_{i\to\infty}\left\|
    \frac{1}{\mu}\mathbb{E}\!\left[
    \widetilde{\boldsymbol{{\scriptstyle \mathcal{W}}}}_{i}
    \widetilde{\boldsymbol{{\scriptstyle \mathcal{W}}}}_{i}^{\top}
    \right]-\Pi
    \right\|^2=0,
\end{equation}
where
\[
\widetilde{\boldsymbol{{\scriptstyle \mathcal{W}}}}_{i}
\triangleq
{\scriptstyle \mathcal{W}}^o-\boldsymbol{{\scriptstyle \mathcal{W}}}_{i} .
\]
That is, the steady-state error covariance scales as $\mu\Pi$, where $\Pi \in \mathbb{R}^{MK\times MK}$ is a symmetric positive semidefinite matrix characterized as the unique solution of a continuous Lyapunov equation~\cite{Clustering}.
\end{lemma}

To learn the graph Laplacian from the noisy estimates $\boldsymbol{w}_{k,i-1}$, the global maximum likelihood (GML) estimator proposed in~\cite{wan} is infeasible, since it requires centralized information and a global matrix inversion whose complexity scales cubically with the network size.
These limitations motivate the use of a \emph{decentralized} Laplacian estimation strategy, whereby each node infers only a local portion of the graph structure.

Let $\mathcal{N}_k^c \triangleq \mathcal{V}\setminus \mathcal{N}_k$ be the complement of the neighborhood of node \( k \), and let $\mathcal{S}_k$ denote the neighborhood without the node itself $\mathcal{S}_k\triangleq\mathcal{N}_k\setminus \{k\}$. Define $\Sigma_{\mathcal{N}_k,\mathcal{N}_k}$ as the marginal covariance matrix associated with node~$k$ and its neighborhood, obtained as the corresponding principal submatrix of the global covariance matrix~$\Sigma$.

By applying the Schur complement identity~\cite{schur}, the inverse of this marginal covariance can be expressed as
\begin{align}
\Gamma_k
\triangleq
\Sigma^{-1}_{\mathcal{N}_k,\mathcal{N}_k}
&=
L_{\mathcal{N}_k,\mathcal{N}_k}
-
\begin{bmatrix}
0 \\[2pt]
L_{\mathcal{S}_k,\mathcal{N}_k^c}
\end{bmatrix}
L_{\mathcal{N}_k^c,\mathcal{N}_k^c}^{-1}
\begin{bmatrix}
0,\; L_{\mathcal{N}_k^c,\mathcal{S}_k}
\end{bmatrix}
\nonumber\\[4pt]
&=
L_{\mathcal{N}_k,\mathcal{N}_k}
-
\begin{bmatrix}
0 & 0 \\[2pt]
0 &
L_{\mathcal{S}_k,\mathcal{N}_k^c}
\,L_{\mathcal{N}_k^c,\mathcal{N}_k^c}^{-1}\,
L_{\mathcal{N}_k^c,\mathcal{S}_k}
\end{bmatrix},
\label{eq:schur_complement}
\end{align}
where the zero blocks arise because node~$k$ has no direct connections to nodes outside $\mathcal{N}_k$ in the Laplacian structure. In other words, an implicit assumption of the GMRF prior is the conditional independence of nodes which are not connected, which results in the sparsity pattern observed in \eqref{eq:schur_complement}.

A key consequence of~\eqref{eq:schur_complement} is that, due to the sparsity of the Laplacian matrix~$L$, the entries in the first row and column of $\Gamma_k$ coincide exactly with those of the submatrix $L_{\mathcal{N}_k,\mathcal{N}_k}$.
As a result, the local parameters in $L_{k,\mathcal{N}_k}$ can be recovered from purely local covariance information.
These locally inferred quantities are sufficient for subsequent decentralized multitask learning, and similar constructions have been extensively studied in~\cite{DML1,DML2,DML3,graph_learn3}.

In this work, we adopt the marginal maximum likelihood (MML) formulation introduced in~\cite{graph_learn3}.
At node~$k$, the MML problem is given by
\begin{align}
\widehat{\Gamma}_{k}
&=
\operatorname*{arg\,min}_{\Gamma}
\;
\big\langle
\widehat{\Sigma}_{\mathcal{N}_k,\mathcal{N}_k},\, \Gamma
\big\rangle
-
\log\det(\Gamma)
\nonumber\\
\text{s.t.}&\quad
\Gamma \succeq 0,
\label{eq:relaxed_local_MLE}
\end{align}
where $\widehat{\Sigma}_{\mathcal{N}_k,\mathcal{N}_k}$ denotes the local sample covariance matrix, which is a principal submatrix of $\widehat{\Sigma}$. Under a zero-mean model, $\widehat{\Sigma}$ is estimated by an empirical covariance estimator with the steady-state estimates $\boldsymbol{{\scriptstyle \mathcal{W}}}_i$ from recursion~\eqref{eq:GD recursion}:
\begin{equation}
    \widehat{\boldsymbol{\Sigma}}
=
\frac{1}{M}\,
\operatorname{Tr_K}\!\big(
(P\,{\scriptstyle \boldsymbol{\mathcal{W}}}_i)
(P\,{\scriptstyle \boldsymbol{\mathcal{W}}}_i)^\top
\big).
\end{equation}
where $\operatorname{Tr_K}(\cdot)$ denotes the block-trace operator that sums the $K\times K$ diagonal blocks, and $P\in\mathbb{R}^{MK\times MK}$ is the commutation matrix that converts the agent-wise stacking $\boldsymbol{{\scriptstyle \mathcal{W}}}_i=\operatorname{col}\{\boldsymbol{w}_{1,i},\ldots,\boldsymbol{w}_{K,i}\}$ into an entry-wise stacking, so that the same coordinate from different agents is grouped together. In particular,
\begin{equation}
P\boldsymbol{{\scriptstyle \mathcal{W}}}_i
=\operatorname{vec}\left(\,[\boldsymbol{w}_{1,i},\ldots,\boldsymbol{w}_{K,i}]^{\top}\right).
\end{equation}
Equivalently, $\widehat{\Sigma}_{\mathcal{N}_k,\mathcal{N}_k}$ can be estimated directly from its local stack without any global operation.

The solution to~\eqref{eq:relaxed_local_MLE} admits the closed form:
\begin{equation}
\widehat{\boldsymbol{\Gamma}}_{k}
=
\widehat{\boldsymbol{\Sigma}}_{\mathcal{N}_k,\mathcal{N}_k}^{-1},
\end{equation}
The Laplacian estimate is then assembled by extracting the entries consistent with the known local connectivity, thereby preserving the sparsity pattern of the Laplacian. To enforce symmetry, we follow~\cite{graph_learn3} and apply a simple local averaging step:
\begin{equation}
\left\{
\begin{aligned}
&\widehat{\boldsymbol{L}}^{\mathrm{\prime}}_{k,\mathcal{N}_k}
=
\operatorname{S}(\widehat{\boldsymbol{\Gamma}}_{k}),\\
&\widehat{\boldsymbol{L}}^{\mathrm{\prime}}_{k,\ell}=0,
\quad \ell\notin \mathcal{N}_k,
\end{aligned}
\right.
\qquad
\overline{\boldsymbol{L}}^\prime
=
\frac{1}{2}
\bigl(
\widehat{\boldsymbol{L}}^{\mathrm{\prime}}
+
(\widehat{\boldsymbol{L}}^{\mathrm{\prime}})^\top
\bigr),
\label{eq:L estimate constraints}
\end{equation}
where $\operatorname{S}(\cdot)$ denotes the subselection operator that extracts the first-row entries corresponding to node~$k$.

The estimator in
\eqref{eq:relaxed_local_MLE}--\eqref{eq:L estimate constraints}
has been shown to be asymptotically consistent under ideal observations
\cite{graph_learn3,DML3}. In the present setting, however, the available
samples are corrupted by noise induced by the stochastic learning
dynamics in~\eqref{eq:GD recursion}. Consequently, the classical
asymptotic consistency arguments do not apply directly, and the
finite-sample estimate $\overline{\boldsymbol{L}}^{\prime}$ obtained
from~\eqref{eq:L estimate constraints} is not necessarily a valid GGL.
We therefore define the final Laplacian estimate by projecting
$\overline{\boldsymbol{L}}^{\prime}$ onto the set of admissible GGL
matrices:
\begin{equation}
\widehat{\boldsymbol{L}}
\triangleq
\mathcal{P}_{\mathcal{C}_{\mathcal{G}}}
\left(\overline{\boldsymbol{L}}^{\prime}\right)
=
\operatorname*{arg\,min}_{X\in\mathcal{C}_{\mathcal{G}}}
\left\|X-\overline{\boldsymbol{L}}^{\prime}\right\|_{F}^{2},
\label{eq:GGL_projection}
\end{equation}
where $\mathcal{C}_{\mathcal{G}}$ denotes the closed convex set of valid
GGL matrices compatible with the known topology $\mathcal{G}$.

An overall procedure is summarized in Algorithm~\ref{alg:alg1}. To state the decentralized implementation, define the local stack \(
\boldsymbol{{\scriptstyle \mathcal{W}}}_{\mathcal{N}_k,i}
\triangleq
\operatorname{col}\{
\boldsymbol{w}_{\ell,i};\,
\ell\in\mathcal{N}_k
\},
\) which collects the estimates available to agent~$k$, and let $P_k$ denote the associated commutation matrix.  
In Phase~II, each agent forms an intermediate iterate $\boldsymbol{\psi}_{k,i}$ and updates $\boldsymbol{w}_{k,i}$ using only the intermediate iterates received from its neighborhood $\mathcal{N}_k$.

\begin{algorithm}[ht]
\caption{Decentralized Multitask Learning over Learned Task Graphs}\label{alg:alg1}
\begin{algorithmic}
\STATE \textbf{Input:} For each agent $k$: neighborhood $\mathcal{N}_k$ and data $\boldsymbol{x}_{k,i}$; initialized $w_{k,0}$; stepsizes $\mu$ and $\mu'$;  regularization strength $\eta$; iteration numbers $N_1$ and $N_2$.
\STATE
\STATE {\textsc{Phase I (Non-cooperative)}}
\STATE $ \textbf{for iterations } i=1,...,N_1\textbf{ do}$
\STATE \hspace{0.5cm}$ \textbf{for agents } k=1,...,K\textbf{ do}$
\vspace{0.1cm}
\STATE \hspace{1cm}$\widehat{\nabla J}_k(\boldsymbol{w}_{k,i-1}) \triangleq \nabla Q_k(\boldsymbol{w}_{k,i-1};\boldsymbol{x}_{k,i})$
\STATE \hspace{1cm}$\boldsymbol{w}_{k,i}
=
\boldsymbol{w}_{k,i-1}
-
\mu\,\widehat{\nabla J}_k(\boldsymbol{w}_{k,i-1})$
\STATE $ \textbf{for agents } k=1,...,K\textbf{ do}$
\STATE \hspace{0.5cm}$\widehat{\boldsymbol{\Gamma}}_{k}=\bigl(
\frac{1}{M}\,
\operatorname{Tr_{|\mathcal{N}_k|}}\!\big(
(P_k\,\boldsymbol{{\scriptstyle \mathcal{W}}}_{\mathcal{N}_k,N_1})
(P_k\,\boldsymbol{{\scriptstyle \mathcal{W}}}_{\mathcal{N}_k,N_1})^\top
\big)\bigr)^{-1}$
\vspace{0.1cm}
\STATE \hspace{0.5cm}$\widehat{\boldsymbol{L}}^{\mathrm{\prime}}_{k,\mathcal{N}_k}
=
\operatorname{S}(\widehat{\boldsymbol{\Gamma}}_{k})$
\STATE \hspace{0.5cm}$\textbf{if }
 \ell\notin \mathcal{N}_k, \quad  \widehat{\boldsymbol{L}}^{\mathrm{\prime}}_{k,\ell}=0 $ 
\STATE \hspace{0.5cm}$\widehat{\boldsymbol{L}}= \mathcal{P}_{\mathcal{C}_{\mathcal{G}}}\left(\frac{1}{2}
\left(
\widehat{\boldsymbol{L}}^{\mathrm{\prime}}
+
(\widehat{\boldsymbol{L}}^{\mathrm{\prime}})^\top
\right)\right)$
 
\STATE

\STATE {\textsc{Phase II (Cooperative)}}
\STATE $\textbf{for iterations } i=N_1+1,\ldots,N_2 \textbf{ do}$
\STATE \hspace{0.5cm}
$\textbf{for agents } k=1,\ldots,K \textbf{ do}$
\STATE \hspace{1cm}
$\widehat{\nabla J}_k(\boldsymbol{w}_{k,i-1})
\triangleq
\nabla Q_k(\boldsymbol{w}_{k,i-1};\boldsymbol{x}_{k,i})$
\STATE \hspace{1cm}
$\boldsymbol{\psi}_{k,i}
\triangleq
\boldsymbol{w}_{k,i-1}
-\mu'\widehat{\nabla J}_k(\boldsymbol{w}_{k,i-1})$
\STATE \hspace{0.5cm}
$\textbf{for agents } k=1,\ldots,K \textbf{ do}$
\STATE \hspace{1cm}
$\displaystyle
\boldsymbol{w}_{k,i}
=
\boldsymbol{\psi}_{k,i}
-\mu'\eta
\sum_{\ell\in\mathcal{N}_k}
\widehat{\boldsymbol{L}}_{k,\ell}
\boldsymbol{\psi}_{\ell,i}$

\textbf{Return:} $\{\boldsymbol{w}_{k,N_2}\}_{k=1}^K$
\end{algorithmic}
\label{alg1}
\end{algorithm}


In the following sections, we are going to derive explicit finite-sample
error bounds that characterize the quality of the resulting
Laplacian estimates and their impact on decentralized
multitask learning.

\section{Estimation Quality Analysis}
\noindent To derive analytical guarantees, we introduce the following assumptions on the cost $J_k(\cdot)$ and the gradient noise process 
$\boldsymbol{{\scriptstyle \mathcal{S}}}_{i}(\cdot)$, defined as
\begin{equation}
    \boldsymbol{{\scriptstyle \mathcal{S}}}_{i}(\boldsymbol{{\scriptstyle \mathcal{W}}}_{i-1})
    = \widehat{\nabla \mathcal{J}}(\boldsymbol{{\scriptstyle \mathcal{W}}}_{i-1})
    - \nabla \mathcal{J}(\boldsymbol{{\scriptstyle \mathcal{W}}}_{i-1}).
\end{equation}
These assumptions are commonly satisfied by objective functions of interest in learning and adaptation, such as quadratic and logistic costs \cite{Adaptation}.

\begin{assumption}[Cost functions]\label{assumption:Risk function} 1) Each individual cost $J_k(w_k)$ is assumed to be convex, twice differentiable, and with bounded Hessian satisfying:
\begin{equation}
     \nu I_M\leq\nabla^{2}J_k(w_k)\leq\delta I_M,
\end{equation}
where $0<\nu\leq\delta <\infty $.

2) The Hessian $\nabla^{2}J_k(\cdot)$ satisfies the Lipschitz condition for any $w_1$, $w_2\in\mathbb{R}^M$ and $k_H\geq0$:
\begin{equation}
    \|\nabla^{2}J_k(w_1)-\nabla^{2}J_k(w_2)\|\leq k_H\|w_1-w_2\|.
\end{equation}
\end{assumption} 

\begin{assumption}[Gradient noise]\label{assumption:Gradient noise} 1) For any $\boldsymbol{{\scriptstyle \mathcal{W}}}_{i-1}$, gradient noise satisfies:
\begin{align}
    \mathbb{E}[\boldsymbol{{\scriptstyle \mathcal{S}}}_{i}(\boldsymbol{{\scriptstyle \mathcal{W}}}_{i-1})|\boldsymbol{{\scriptstyle \mathcal{W}}}_{i-1}] &= 0\\
     \mathbb{E}[\|\boldsymbol{{\scriptstyle \mathcal{S}}}_{i}(\boldsymbol{{\scriptstyle \mathcal{W}}}_{i-1})\|^4|\boldsymbol{{\scriptstyle \mathcal{W}}}_{i-1}] &\leq \beta^4\|{\scriptstyle \mathcal{W}}^o-\boldsymbol{{\scriptstyle \mathcal{W}}}_{i-1}\|^4+\sigma_s^4,
\end{align}
where $\beta,\, \sigma_s \geq0.$

2) The conditional covariance of $\boldsymbol{{\scriptstyle \mathcal{S}}}_{i}(\boldsymbol{{\scriptstyle \mathcal{W}}}_{i-1})$ defined as  
\[\mathcal{R}_{s,i}(\boldsymbol{{\scriptstyle \mathcal{W}}}_{i-1})\triangleq \mathbb{E}[\boldsymbol{{\scriptstyle \mathcal{S}}}_{i}(\boldsymbol{{\scriptstyle \mathcal{W}}}_{i-1})\boldsymbol{{\scriptstyle \mathcal{S}}}^\top_{i}(\boldsymbol{{\scriptstyle \mathcal{W}}}_{i-1})|\boldsymbol{\mathcal{F}}_{i-1}]\] Where $\boldsymbol{\mathcal{F}}_{i-1}$denotes the filtration corresponding to all past iterates across all agents. The conditional covariance satisfies:
\begin{align}
        \|\mathcal{R}_{s,i}(\boldsymbol{{\scriptstyle \mathcal{W}}}_{i-1})-\mathcal{R}_{s,i}({\scriptstyle \mathcal{W}}^o)\| &\leq k_s\|\boldsymbol{{\scriptstyle \mathcal{W}}}_{i-1}-{\scriptstyle \mathcal{W}}^o\|^{\gamma_s}\\
         \mathcal{R}_s \triangleq \lim_{i \to \infty}\mathcal{R}_{s,i}({\scriptstyle \mathcal{W}}^o) &> 0,
\end{align}
where $k_s\geq 0$ and $0<\gamma_s\leq4$.
\end{assumption} 

\subsection{Laplacian Estimation Error}
We next characterize the accuracy of the learned Laplacian and its dependence on the model dimension \(M\), learning dynamics, and graph topology.
\begin{theorem} [Laplacian estimation error]\label{Theorem: Laplacian estimation error} Suppose Assumptions~\ref{assumption:GMRF} through~\ref{assumption:Gradient noise} hold, and the steady-state covariance satisfies
\(
\Phi\triangleq P\Pi P^\top
=
\operatorname{diag}\{\Phi_1,\ldots,\Phi_M\}
\), where each $\Phi_m\succeq0$ is deterministic with respect to
${\scriptstyle \mathcal{W}}^o$. 
Then, for a sufficiently small stepsize $\mu$ and sufficiently large $M$,
the Laplacian estimator $\widehat{\boldsymbol{L}}$ satisfies
\begin{align}
&\mathbb{E}_{\mathcal{W}^o}\left[\mathbb{E}_{\boldsymbol{{\scriptstyle \mathcal{W}}}_i\mid {\scriptstyle \mathcal{W}}^o}\,
\|\widehat{\boldsymbol{L}}-L\|^{2}\right]
\leq 
\sum_{k=1}^{K}\frac{\|L_{\mathcal{N}_k,\mathcal{N}_k}\|^{4}}{(1-\zeta)^{2}}\;\epsilon,
\label{eq: L inequality}\\
&\text{where}\quad\epsilon \triangleq\;
\underbrace{
3\left(\tau\|\Phi\|^2+\rho\left\|L^{-1}\right\|^2\right)
\left(\frac{K}{M}+\frac{K^2}{M^2}\right)
}_{\text{covariance concentration}}\nonumber\\
&\hspace{5em}+\;
\underbrace{
12\mu\operatorname{Tr}(L^{-1})
\operatorname{Tr}\left(\frac{\Phi}{M}\right)
+3\mu^2\Bigl\|
\operatorname{Tr_K}\left(\frac{\Phi}{M}\right)
\Bigr\|^2
}_{\text{learning-induced variance}}\label{eq:covariance error}.
\end{align}
Here, $0<\zeta<1$ is a constant, $\rho$ and $\tau$ are nonnegative constants.
\end{theorem}
\begin{proof}
See Appendix (Proof of Theorem~\ref{Theorem: Laplacian estimation error}).
\end{proof}
The error bound in Theorem~\ref{Theorem: Laplacian estimation error} characterizes how the quality of the learned Laplacian depends jointly on the stepsize $\mu$ and the dimension $M$. The bound decomposes naturally into two components. The first term in $\epsilon$, labeled as \emph{covariance concentration}, captures statistical fluctuations due to finite dimensions and decays at a rate $O\!\left(\tfrac{K}{M}\right)$, consistent with standard covariance concentration results. This term vanishes as $M \to \infty$. The \emph{learning-induced variance} arises from the steady-state fluctuations of the non-cooperative stochastic-gradient recursion. It scales as $O(\mu)$ and therefore vanishes as $\mu\to0$. This means that to reduce the overall error in Laplacian estimation, one would need to increase the dimension \( M \), while reducing the stepsize \( \mu \), or in other words:
\begin{equation}
    \lim_{\mu \to 0}\;\lim_{M \to \infty}\, \mathbb{E}_{\boldsymbol{{\scriptstyle \mathcal{W}}}_i\mid {\scriptstyle \mathcal{W}}^o}\,\|\widehat{\boldsymbol{L}}-L\|^{2} = 0 \label{eq:zero bias}.
\end{equation}
Moreover, we define a quantity 
\begin{equation}
  \mathrm{TSI}(L) \triangleq \sum_{k=1}^{K}\|L_{\mathcal{N}_k,\mathcal{N}_k}\|^4,  
\end{equation}
which serves as a \emph{topology sensitivity index} that characterizes how the network structure influences the Laplacian estimation error. Each term $\|L_{\mathcal{N}_k,\mathcal{N}_k}\|$ measures the magnitude of the local Laplacian block around node~\(k\), which generally increases with the node degree and the weights of its local connections. Consequently, larger values of $\mathrm{TSI}(L)$ indicate that small perturbations in local covariance estimates are more strongly amplified in the reconstruction of the global Laplacian. In particular, heterogeneous networks with hub nodes tend to exhibit larger $\mathrm{TSI}(L)$ and hence higher estimation sensitivity. We will illustrate this finding in the numerical experiments below.
\subsection{Laplacian-Induced Error in Multitask Learning}
\noindent With the Laplacian estimation error characterized in Theorem~\ref{Theorem: Laplacian estimation error}, we now examine its effect on the multitask diffusion recursion in~~\eqref{eq:DGD_recursion}. Our analysis adapts the techniques in~\cite{stefan_reg1} to allow for strictly positive definite and inaccurate estimates of the true Laplacian. This introduces some technical challenges to the analysis, which we resolve here.
\paragraph{Network mean-square deviation.}
We begin by introducing the mean-square deviation (MSD) to quantify how closely the stochastic recursion tracks the optimal solution associated with the Laplacian $\widehat{L}$.
Define the network error vector at time $i$ as
\begin{equation}
    \widetilde{\boldsymbol{{\scriptstyle \mathcal{W}}}}_{\widehat{L},i}
\triangleq
{\scriptstyle \mathcal{W}}^o_{\widehat{L}}
-
\boldsymbol{{\scriptstyle \mathcal{W}}}_{\widehat{L},i}.
\end{equation}

\begin{lemma}[Network MSD performance]\label{lemma:MSD bound}
Let Assumptions~\ref{assumption:GMRF} through~\ref{assumption:Gradient noise} hold.
For a sufficiently small stepsize $\mu^\prime$, the steady-state network MSD satisfies
\begin{align}
\limsup_{i \to \infty}\mathbb{E}\|\widetilde{\boldsymbol{{\scriptstyle \mathcal{W}}}}_{\widehat{L},i}\|^2
&\leq \frac{\mu^{\prime}\operatorname{Tr}\left((I_{MK}-\mu^\prime\eta\widehat{\mathcal{L}})^2\mathcal{R}_s\right)}{\min\{2\nu-\mu^{\prime}\nu^2,\notag\;\;2\delta-\mu^{\prime}\delta^2\}}+\operatorname{O}\left(\mu^{\prime}\right)\notag\\
&=O(\mu').\label{Network MSD performance}
\end{align}
\end{lemma}

\begin{proof}
See Appendix (Proof of Lemma~\ref{lemma:MSD bound}).
\end{proof}

\paragraph{Laplacian-induced bias}
When the diffusion recursion operates with an estimated Laplacian $\widehat{L}$, an additional deterministic bias arises due to the mismatch between the true and learned graph structures.
To quantify this effect, define
\[
\Delta{\scriptstyle \mathcal{W}}_{\widehat{L}}^{o}
\triangleq
{\scriptstyle \mathcal{W}}^{o}_{L}
-
{\scriptstyle \mathcal{W}}^{o}_{\widehat{L}},
\]
which represents the deviation between the optimal solutions associated with $L$ and $\widehat{L}$.
The \emph{Laplacian-induced bias} is defined as
\(
\mathbb{E}_{\boldsymbol{{\scriptstyle \mathcal{W}}}_i\mid {\scriptstyle \mathcal{W}}^o}
\|
\Delta{\scriptstyle \mathcal{W}}_{\widehat{L}}^{o}
\|^{2},
\)
where the expectation is taken with respect to the randomness in the estimated Laplacian.

\begin{lemma}[Laplacian-induced bias]\label{lemma:Laplacian-induced bias}
Suppose the conditions of Lemma~\ref{lemma:MSD bound} hold.
Then,
\begin{equation}
\bigl\|
\Delta{\scriptstyle \mathcal{W}}_{\widehat{L}}^{o}
\bigr\|^{2}
\le
\frac{\eta^{2}}{\lambda_{\widehat{L}}^{2}}\,
\|\widehat{L}-L\|^{2}\,
\|{\scriptstyle \mathcal{W}}^{o}_{L}\|^{2},
\end{equation}
where
\(
\lambda_{\widehat{L}}
\triangleq
\nu + \eta\,\lambda_{\min}(\widehat{L}).
\)
Consequently, the conditional Laplacian-induced
bias satisfies
\begin{equation}
    \mathbb{E}_{\boldsymbol{{\scriptstyle \mathcal{W}}}_i
    \mid{\scriptstyle \mathcal{W}}^o}
    \left\|
    \Delta{\scriptstyle \mathcal{W}}_{\widehat L}^{o}
    \right\|^2
    \leq
    \frac{\eta^2}{\lambda_{\widehat{L}}^{2}}
    \left\|{\scriptstyle \mathcal{W}}_L^o\right\|^2
    \mathbb{E}_{\boldsymbol{{\scriptstyle \mathcal{W}}}_i
    \mid{\scriptstyle \mathcal{W}}^o}
    \left\|\widehat L-L\right\|^2.
    \label{eq:fixed_M_bias}
    \end{equation}
\end{lemma}

\begin{proof}
See Appendix (Proof of Lemma~\ref{lemma:Laplacian-induced bias}).
\end{proof}

\paragraph{Overall Laplacian-induced deviation}
Combining Lemmas~\ref{lemma:MSD bound} and
\ref{lemma:Laplacian-induced bias}, we characterize the deviation of
$\boldsymbol{{\scriptstyle \mathcal{W}}}_{\widehat{L},i}$
from the true optimal
solution ${\scriptstyle \mathcal{W}}^o_L$. 

\begin{theorem}[Laplacian-induced deviation]\label{eq:Laplacian-induced deviation}
The steady-state deviation of the diffusion recursion with an estimated Laplacian satisfies
\begin{align}
&
\limsup_{i\to\infty}
\mathbb{E}\,
\left\|
{\scriptstyle \mathcal{W}}^{o}_{L}
-
\boldsymbol{{\scriptstyle \mathcal{W}}}_{\widehat{L},i}
\right\|^{2}
\notag\\
&\;\;=
\limsup_{i\to\infty}
\mathbb{E}\,
\bigl\|
\Delta{\scriptstyle \mathcal{W}}_{\widehat{L}}^{o}
+
\widetilde{\boldsymbol{{\scriptstyle \mathcal{W}}}}_{\widehat{L},i}
\bigr\|^{2}
\nonumber\\
&\hspace{-0.18em}\overset{\eqref{eq:covariance error},\eqref{Network MSD performance}}{=}
\operatorname{O}\left(
    \frac{K}{M}+
\mu\operatorname{Tr}(L^{-1})
\operatorname{Tr}\left(\frac{\Phi}{M}\right)
\right)+
\operatorname{O}\left(\mu^2\right)
\notag\\
&\hspace{4em}
+\operatorname{O}\left(
\mu^{\prime}
\frac{
\operatorname{Tr}\left(
\left(I_{MK}-\mu^\prime\eta\widehat{\mathcal{L}}\right)^2
\mathcal{R}_s
\right)
}{
\min\left\{
2\nu-\mu^\prime\nu^2,\,
2\delta-\mu^\prime\delta^2
\right\}
}
\right)\\
&\;\;=
O(\frac{1}{M}+\mu) + O(\mu').
\end{align}
\end{theorem}
Theorem~\ref{eq:Laplacian-induced deviation} shows that the steady-state error of multitask diffusion with a learned graph decomposes into two distinct components: a stochastic error term of order $O(\mu')$ due to gradient noise, and a variance of order $O(\frac{1}{M}+\mu)$ stemming from graph estimation errors.
As the dimension $M$ grows and the stepsize $\mu,\,\mu^{\prime}$ decreases, both effects vanish, ensuring that the diffusion recursion asymptotically recovers the optimal solution associated with the true Laplacian, which in turn carries an interpretation of a maximum a posteriori estimate~\cite{stefan_reg2}.

\section{NUMERICAL EXPERIMENTS}
\subsection{General Settings}
\noindent On a connected, undirected network, each node is subjected to streaming data $\{\boldsymbol{d}_k(i), \boldsymbol{u}_{k,i}\}$, satisfying a linear regression model:
\begin{equation}
    \boldsymbol{d}_k(i) = \boldsymbol{u}^\top_{k,i}w^o_k + \boldsymbol{v}_k(i),\; k=1,...,K.
\end{equation}
The processes $\{\boldsymbol{u}_{k,i}, \boldsymbol{v}_k(i)\}$ are zero-mean jointly wide-sense stationary with: i) $\mathbb{E}\boldsymbol{u}_{k,i}\boldsymbol{u}^\top_{\ell,i}=\sigma^2_{u,k}I_M$ if $k=\ell$ and zero otherwise; ii) $\mathbb{E}\boldsymbol{v}_{k}(i)\boldsymbol{v}_{\ell}(i)=\sigma^2_{v,k}$ if $k=\ell$ and zero otherwise; iii) $\boldsymbol{u}_{k,i}$ and $\boldsymbol{v}_k(i)$ are independent. According to \eqref{eq:cost_function}, the cost functions take the mean-square-error form:

\begin{equation}
    J_k(w_k) = \frac{1}{2}\mathbb{E}|{\boldsymbol{d}}_k(i)-\boldsymbol{u}^\top_{k,i}w_k|^2.
\end{equation}

\subsection{Laplacian Learning Errors}
\begin{figure}[t]
    \centering
    \includegraphics[width=0.7\linewidth]{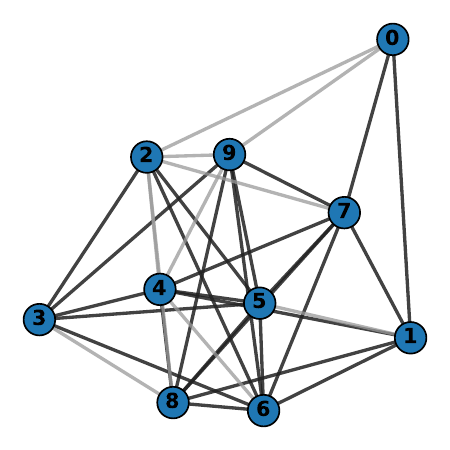}
    \caption{Graph topology with assigned edge weights. Grey edges correspond to weight $1$, while black edges correspond to weight $5$.}
    \label{fig:graph}
\end{figure}

\begin{figure}[t]
    \centering
    \includegraphics[width=0.8\linewidth]{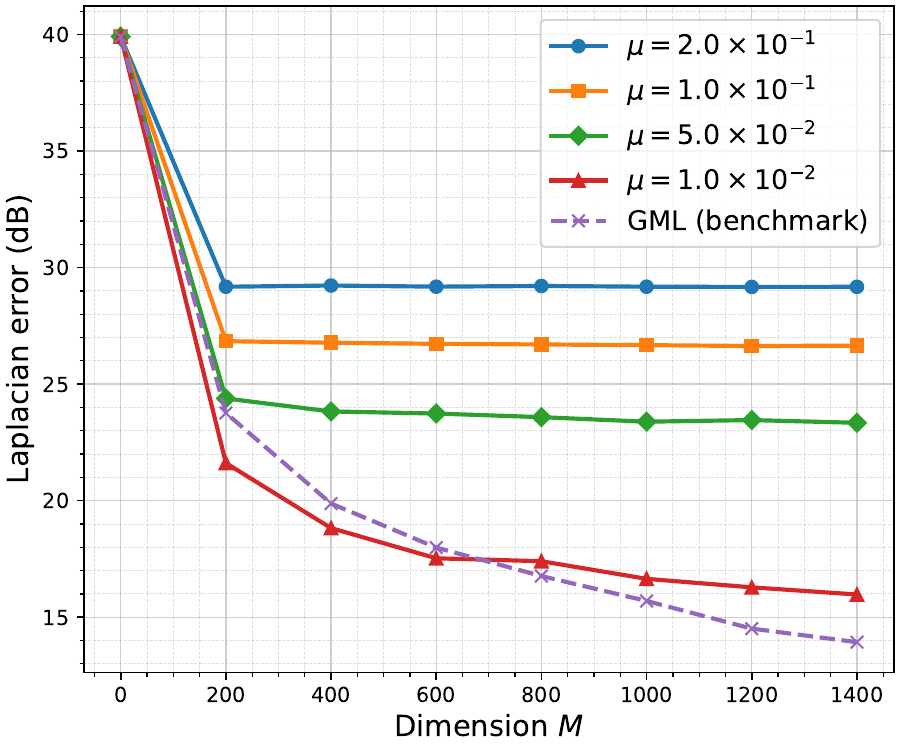}
    \caption{Avg. Laplacian estimation error $\|\widehat L-L\|^2$ versus $M$.}
    \label{fig:L error}
\end{figure}

\noindent \paragraph{Impacts of stepsize $\mu$ and dimension $M$} The Laplacian estimation is conducted on a network with $K=10$ nodes and maximum node degree $8$, illustrated in Fig.~\ref{fig:graph}. To obtain an unbalanced weighted topology that reflects heterogeneous task relationships and poses a more challenging scenario for multitask learning, each edge weight is randomly drawn from the set $\{1,\,5\}$. The self-loop matrix is set as the identity matrix $V=I_K$.

We run the recursion in~\eqref{eq:GD recursion} with different stepsizes
$\mu \in \{2\times10^{-1},\,1\times10^{-1},\,5\times10^{-2},\,10^{-2}\}$
for $10000$ iterations. The resulting steady-state iterates $\scriptstyle \mathcal{W}_i$ are then used in~\eqref{eq:relaxed_local_MLE}--\eqref{eq:L estimate constraints} to construct an estimate $\widehat{L}$ of the true Laplacian $L$. The corresponding Laplacian estimation errors $\|\widehat{L}-L\|^2$, averaged over $100$ Monte-Carlo realizations of $\boldsymbol{\scriptstyle \mathcal{W}}_i$, are reported in Fig.~\ref{fig:L error}. The resulting marginal-based Laplacian estimates are further compared against a benchmark GML estimator obtained using the true parameter vector $\boldsymbol{\scriptstyle \mathcal{W}}^o$.

We observe that when $M$ is sufficiently large, the estimation error converges to the bias term characterized in~\eqref{eq:covariance error}, scaling consistently with the theoretical rate $O(\mu)$. This behavior is evident in Fig.~\ref{fig:L error}; for example, at $M=800$, halving the stepsize $\mu$ reduces the error magnitude by approximately a factor of $1/2$, corresponding to a decrease of about $3$\,dB in the error curve. 

Conversely, when $\mu$ is very small and $M$ is not sufficiently large, the error is dominated by the covariance concentration term, which scales as $O\!\left(\tfrac{K}{M}\right)$. This explains why the curve corresponding to $\mu=10^{-2}$ nearly overlaps with the benchmark before $M=800$, thereby confirming the theoretical prediction.

\paragraph{Impacts of topologies} We next investigate how the network topology influences the Laplacian estimation performance. To isolate the effect of topology, we construct three graphs with identical number of nodes and identical average degree: i) a regular graph with uniform degree distribution, ii) an Erd\H{o}s--R\'enyi (ER) random graph with moderate degree variability, and iii) a hub-dominated graph exhibiting strong degree heterogeneity. All edges are assigned unit weights, and the same self-loop matrix $V=I_K$ is used across all topologies to ensure a fair comparison. The resulting graph structures and the calculated $TSI(L)$ are illustrated in Fig.~\ref{fig:Topologies}.
\begin{figure*}[!t]
    \centering
    \includegraphics[width=0.9\linewidth]{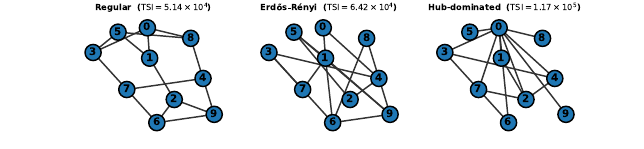}
    \caption{Illustration of three graph topologies with identical average degree: regular, Erd\H{o}s--R\'enyi, and hub-dominated.}
    \label{fig:Topologies}
\end{figure*}

\begin{figure}[t]
    \centering
    \includegraphics[width=0.8\linewidth]{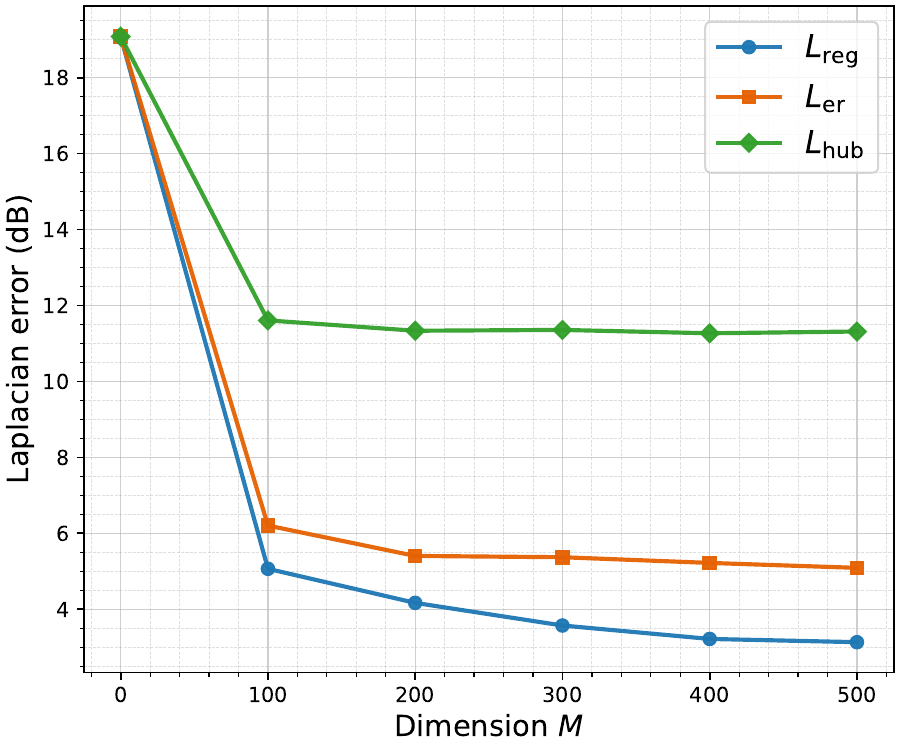}
    \caption{Avg. Laplacian estimation error $\|\widehat L-L\|^2$ versus dimension $M$ under different graph topologies.}
    \label{fig:L_error_topology}
\end{figure}
The recursion in~\eqref{eq:GD recursion} is run with $\mu=5\times10^{-2}$, and the results are applied to estimate graph Laplacians. Fig.~\ref{fig:L_error_topology} reports the average Laplacian estimation error for the three topologies. Consistent with the theoretical analysis, the hub-dominated graph yields the largest estimation error across all values of $M$. This behavior is attributed to the presence of a high-degree hub node, which produces a large local Laplacian block and substantially increases $\mathrm{TSI}(L)$. As a result, perturbations in the local covariance estimates are more strongly amplified during the reconstruction of the global Laplacian.

Overall, these results suggest that the degree heterogeneity plays a critical role in decentralized Laplacian estimation. Networks with highly uneven connectivity, such as hub-dominated graphs, significantly amplify estimation error even when the average connectivity level is kept constant.

\subsection{Steady-State Performance under Diffusion and Learned Laplacians}
\begin{figure}[t]
    \centering
    \includegraphics[width=0.8\linewidth]{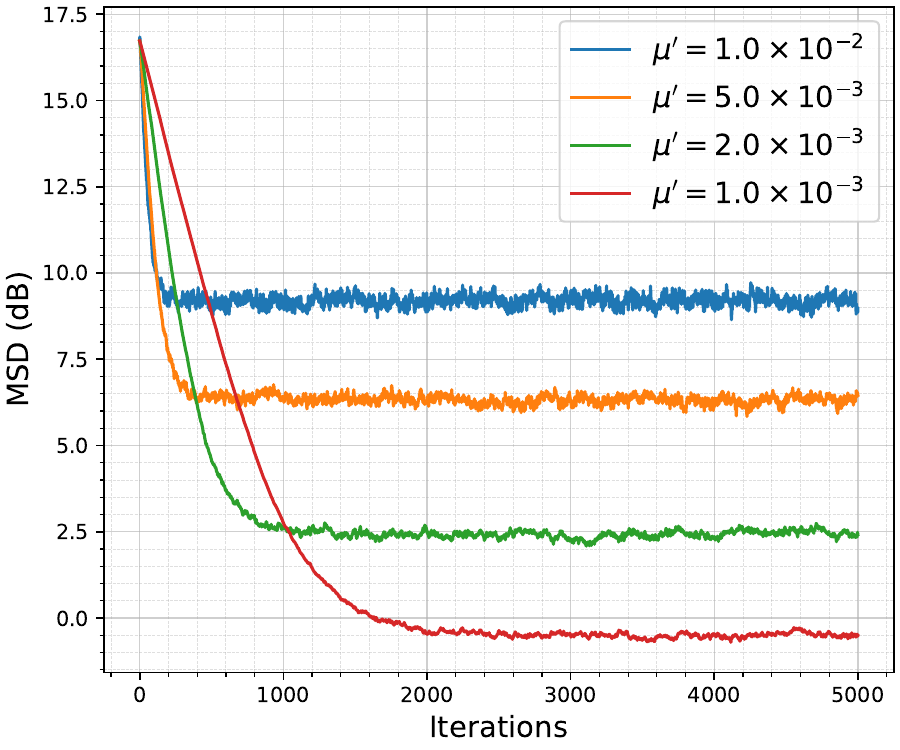}
    \caption{Avg. MSD $\|
\widetilde{{\scriptstyle \mathcal{W}}}_{L,i}
\|^{2}$ with true graph Laplacian.}
    \label{fig:MSD}
\end{figure}

\noindent\paragraph{Impacts of multitask learning stepsizes} We consider the same graph topology shown in Fig.~\ref{fig:graph}. 
The diffusion multitask algorithm in~\eqref{eq:DGD_recursion} is employed to evaluate the steady-state MSD when the \emph{true} Laplacian is available. 
The regularization weight is set to $\eta=1$. 
Under this setting, any performance degradation due to Laplacian estimation is eliminated, thereby isolating the stochastic adaptation behavior of the diffusion recursion. 
The algorithm~\eqref{eq:DGD_recursion} is executed with several diffusion stepsizes $\mu^\prime$, while all other parameters are kept unchanged.

As illustrated in Fig.~\ref{fig:MSD}, the steady-state MSD decreases monotonically with \(\mu^\prime\), consistent with the predicted $O\!\left(\mu^\prime\right)$ scaling and earlier results in~\cite{stefan_reg1}. 

\begin{figure}[t]
    \centering
    \includegraphics[width=0.8\linewidth]{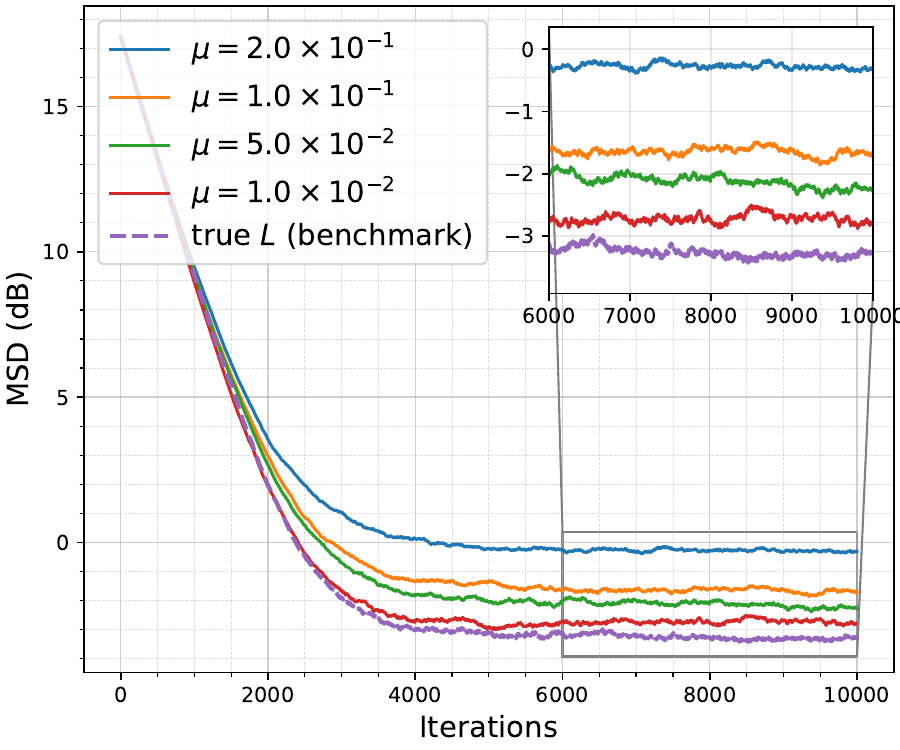}
    \caption{Avg. MSD $\|
\widetilde{{\scriptstyle \mathcal{W}}}_{L,i}
\|^{2}$ with estimated graph Laplacians.}
    \label{fig:MSD_L}
\end{figure}

\paragraph{Impacts of Laplacian learning accuracy}Next, we evaluate the impact of Laplacian learning accuracy on the steady-state performance. 
The same experimental setting is adopted, except that the multitask algorithm now operates with \emph{learned} Laplacians obtained using different non-cooperative learning stepsizes $\mu$, and the multitask learning stepsize is fixed at $\mu^\prime=5\times10^{-4}$. 
Since larger $\mu$ leads to higher Laplacian estimation error, this experiment illustrates how the Laplacian-induced bias affects the final multitask solution.

The resulting MSD trajectories are shown in Fig.~\ref{fig:MSD_L}. 
All curves exhibit a similar transient behavior, but converge to different steady-state error floors. 
In particular, smaller Laplacian learning stepsizes $\mu$ yield more accurate Laplacian estimates and therefore lower steady-state MSD. 
When the learning step-size is sufficiently small, the performance approaches that of the true Laplacian benchmark, confirming that the residual performance gap is primarily caused by Laplacian estimation error. 

\subsection{Joint Learning Performance}
\paragraph{Two-phase learning}
To evaluate the effectiveness of the proposed two-phase framework, we amplify each edge weight in Fig.~\ref{fig:graph} by a factor of $10$. This strengthens the Laplacian regularization and enforces stronger smoothness across neighboring tasks. 

In Phase~I, agents operate non-cooperatively with stepsize $\mu\in\{5\times10^{-3},\,2\times10^{-3}\}$. After convergence, the final iterates $\boldsymbol{w}_{k,i}$ are used to estimate the Laplacian matrix. In Phase~II, the learned Laplacian $\widehat{L}$ is incorporated into a cooperative multitask update.

The performance of the two-phase strategy is evaluated in Fig.~\ref{fig:illustration} (Right) using $\overline{\mathrm{MSD}}$,
\begin{equation}
    \overline{\mathrm{MSD}} \triangleq
\begin{cases}
\mathbb{E}\|\boldsymbol{{\scriptstyle \mathcal{W}}}_{i}-{\scriptstyle \mathcal{W}}^o\|^2, & \text{Phase I}, \\
\mathbb{E}\|\boldsymbol{{\scriptstyle \mathcal{W}}}_{\widehat{L},i}-{\scriptstyle \mathcal{W}}^o\|^2, & \text{Phase II},
\end{cases}
\end{equation}
which measures the deviation from the true model parameter ${\scriptstyle \mathcal{W}}^o$.

In Phase~II, a cooperative stepsize $\mu'=\mu$ is applied. The figure shows a clear transition at iteration $5000$, marking the switch from non-cooperative to cooperative learning. After incorporating the learned graph, both joint learning curves converge to a lower steady-state $\overline{\mathrm{MSD}}$ compared to their non-cooperative phase. This demonstrates that cooperation enabled by the learned task graph improves estimation accuracy. The performance gain is particularly pronounced for larger $\mu$, where the stochastic gradient noise in Phase~I is more significant.
\paragraph{Multi-phase learning}
\begin{figure}[t]
    \centering
    \includegraphics[width=0.8\linewidth]{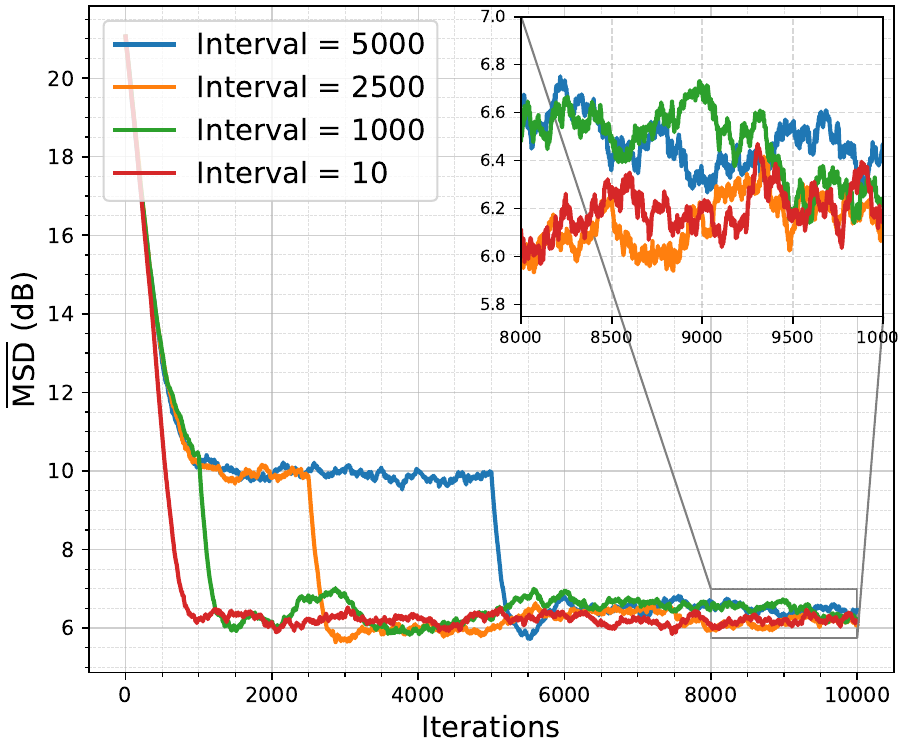}
    \caption{$\overline{\mathrm{MSD}}$ for the multi-phase strategy under different graph Laplacian updating intervals.}
    \label{fig:Multi_phase}
\end{figure}
In practical settings, estimating the graph Laplacian only once after full non-cooperative convergence may be undesirable due to latency or computational cost. We therefore consider a \emph{multi-phase} strategy that periodically updates the Laplacian during learning. After an initial non-cooperative phase, the algorithm alternates between cooperative multitask updates and Laplacian estimation steps every fixed interval, using the most recent parameter iterates.

Since the Laplacian is now learned from iterates that may not be at steady state, the two-phase analysis does not directly apply. Nevertheless, Fig.~\ref{fig:Multi_phase} shows that comparable performance can still be achieved. The setup is identical to the previous experiment: the non-cooperative stepsize is $\mu=5\times10^{-3}$ and the cooperative stepsize is fixed at $\mu^{\prime}=2.7\times10^{-3}$. We test update intervals $\{5000,\,2500,\,1000,\,10\}$. For larger intervals (5000 and 2500), the Laplacian is updated only after the non-cooperative recursion has essentially converged, and the resulting performance closely matches the two-phase benchmark. For shorter intervals (1000 and 10), the Laplacian is updated earlier using noisier iterates, which slightly slows the transient, while leading to nearly the same steady-state $\overline{\mathrm{MSD}}$.

We note, however, that stability depends on the regularization strength $\eta$ and the degree of task smoothness. When the update interval is very short, the Laplacian estimate may be significantly corrupted by its transient error. If, in addition, $\eta$ is chosen large while the true task parameters are weakly correlated (i.e., the smoothness prior is mismatched), the cooperative term can amplify estimation errors rather than suppress them. In this regime, the diffusion recursion may become unstable and the $\overline{\mathrm{MSD}}$ curve can diverge.

\section{Conclusion}

\noindent We proposed a decentralized framework for joint graph learning and multitask adaptation when task relationships are unknown. Under a Gaussian Markov random field model, Laplacian estimation was formulated as a decentralized marginal likelihood problem and analyzed in the presence of stochastic learning noise.

We established explicit finite-sample bounds for the Laplacian estimation error, showing that it decomposes into a covariance concentration term and a steady-state variance that scales with the non-cooperative stepsize. A topology sensitivity quantity was introduced to quantify how network heterogeneity amplifies the estimation error. We further showed that, when the learned Laplacian is used for multitask diffusion, the steady-state deviation consists of an $O(\mu')$ stochastic term and an $O(\mu)$ Laplacian-induced variance compared to the MAP optimal formulation under the known Laplacian.
Simulations validated the theoretical findings and demonstrated that the learned graph enables improved performance over non-cooperative learning while approaching the true-graph benchmark as the stepsizes decrease. The results provide analytical guarantees and design insights for adaptive cooperation in decentralized multitask networks.

\appendix
\section*{Proof of Theorem~\ref{Theorem: Laplacian estimation error}}
\label{proof:L_error}
\addcontentsline{toc}{section}{Proof of Theorem~\ref{Theorem: Laplacian estimation error}}
\noindent Under the Assumption~\ref{assumption:Risk function} and~\ref{assumption:Gradient noise}, we can call on the following Lemma from~\cite{probability1,probability2,Clustering}.
\begin{lemma}[Asymptotic Normality]\label{Theorem: Asymptotic normality}
Under the Assumption~\ref{assumption:Risk function} and~\ref{assumption:Gradient noise}. Consider the recursion~\eqref{eq:GD recursion} with constant stepsize $\mu$. For sufficiently small $\mu$ and as $i\to\infty$, the steady-state estimation error
admits an asymptotically Gaussian approximation in the following sense~\cite{probability1,probability2}:
\begin{equation}
    \frac{\boldsymbol{{\scriptstyle \mathcal{W}}}_i-{\scriptstyle \mathcal{W}}^o}{\sqrt{\mu}}
    \;\Rightarrow\;
    \mathcal{N}(0,\Pi),
\end{equation}
where ``$\Rightarrow$'' denotes convergence in distribution. The matrix $\Pi$ is symmetric positive semidefinite and is independent of $\mu$.
It is characterized as the unique solution to a continuous Lyapunov equation~\cite{Clustering}.

Motivated by the asymptotic normality result, we introduce the
following Gaussian surrogate for the steady-state iterate:
\begin{equation}
    \boldsymbol{{\scriptstyle \mathcal{W}}}_i \mid {\scriptstyle \mathcal{W}}^o
    \;\sim\;
    \mathcal{N}({\scriptstyle \mathcal{W}}^o,\;\mu\,\Pi).
\end{equation}
\end{lemma}

For notational convenience, we introduce an element-wise stacking of the parameter vector by defining
\begin{equation}
\begin{aligned}
       {\scriptstyle \mathcal{W}}^{o,\mathrm{ele}}\triangleq\operatorname{vec}\bigl[W^\top\bigr]&=P\operatorname{vec}\bigl[W\bigr]=P{\scriptstyle \mathcal{W}}^o.
\end{aligned}
\end{equation}
Applying the same transformation to the iterates yields
\(
    \boldsymbol{{\scriptstyle \mathcal{W}}}_i^{\mathrm{ele}}=P\boldsymbol{{\scriptstyle \mathcal{W}}}_i.
\) By Lemma~\ref{Theorem: Asymptotic normality}, the element-wise stacked estimates admit the Gaussian approximation
\begin{align}
\boldsymbol{{\scriptstyle \mathcal{W}}}_i^{\mathrm{ele}} \mid {\scriptstyle \mathcal{W}}^{o,\mathrm{ele}}
    \;&\sim\; \mathcal{N}({\scriptstyle \mathcal{W}}^{o,\mathrm{ele}},\;\mu\,\Phi)\\
\boldsymbol{w}_{m,i}^{\mathrm{ele}} \mid {\scriptstyle \mathcal{W}}^{o,\mathrm{ele}}
 \;&\sim\; \mathcal{N}(w_m^{o,\mathrm{ele}},\;\mu\,\Phi_m), 
 \label{eq:element_distribution}
\end{align}
where $\Phi\triangleq P \Pi P^\top$ and $\Phi_m$ denotes the corresponding marginal covariance associated with the $m$-th element-wise component.

For brevity, $\boldsymbol{w}_m^{\mathrm{ele}}$ denotes the steady-state
version of $\boldsymbol{w}_{m,i}^{\mathrm{ele}}$, with the iteration
index suppressed. Since it is a subvector of
$\boldsymbol{{\scriptstyle \mathcal{W}}}_i$, for any measurable function $f$,
\begin{equation}
\mathbb{E}_{\boldsymbol{{\scriptstyle \mathcal{W}}}_i\mid\mathcal{W}^o}
\left[f(\boldsymbol{w}_m^{\mathrm{ele}})\right]
=
\mathbb{E}_{\boldsymbol{w}_m^{\mathrm{ele}}\mid\mathcal{W}^o}
\left[f(\boldsymbol{w}_m^{\mathrm{ele}})\right].
\end{equation}

\paragraph{The covariance estimation error at node $k$}
We begin by bounding the local marginal covariance:
\begin{align}
&\left\|
\widehat{\boldsymbol{\Sigma}}_{\mathcal{N}_k,\mathcal{N}_k}
-
\Sigma_{\mathcal{N}_k,\mathcal{N}_k}
\right\|^2
\overset{(a)}{\leq}
\left\|
\frac{1}{M}\sum_{m=1}^{M}
\boldsymbol{w}_m^{\mathrm{ele}}
(\boldsymbol{w}_m^{\mathrm{ele}})^\top
-\Sigma
\right\|^2
\notag\\
&\overset{(b)}{=}
\Biggl\|
\frac{1}{M}\sum_{m=1}^{M}
\left(
\boldsymbol{w}_m^{\mathrm{ele}}
(\boldsymbol{w}_m^{\mathrm{ele}})^\top
-
\mathbb{E}\!\left[
\boldsymbol{w}_m^{\mathrm{ele}}
(\boldsymbol{w}_m^{\mathrm{ele}})^\top
\,\middle|\,
{\scriptstyle \mathcal{W}}^o
\right]
\right)
\notag\\
&\hspace{4em}
+
\frac{1}{M}\sum_{m=1}^{M}
\left(
w_m^{o,\mathrm{ele}}
(w_m^{o,\mathrm{ele}})^\top
-\Sigma
\right)
+
\frac{\mu}{M}\sum_{m=1}^{M}\Phi_m
\Biggr\|^2
\notag\\
&\overset{(c)}{\leq}
3\left\|
\frac{1}{M}\sum_{m=1}^{M}
\left(
\boldsymbol{w}_m^{\mathrm{ele}}
(\boldsymbol{w}_m^{\mathrm{ele}})^\top
-
\mathbb{E}\!\left[
\boldsymbol{w}_m^{\mathrm{ele}}
(\boldsymbol{w}_m^{\mathrm{ele}})^\top
\,\middle|\,
{\scriptstyle \mathcal{W}}^o
\right]
\right)
\right\|^2
\notag\\
&\hspace{2em}
+
3\left\|
\frac{1}{M}\sum_{m=1}^{M}
\left(
w_m^{o,\mathrm{ele}}
(w_m^{o,\mathrm{ele}})^\top
-\Sigma
\right)
\right\|^2
+
3\left\|
\frac{\mu}{M}\sum_{m=1}^{M}\Phi_m
\right\|^2.
\label{equ:noise_terms}
\end{align}
Step~(a) follows because the local covariance error is a principal
submatrix of the global covariance error. Step~(b) adds and subtracts the
conditional second moment and uses~\eqref{eq:element_distribution}: 
\(
\mathbb{E}\!\left[
\boldsymbol{w}_m^{\mathrm{ele}}
(\boldsymbol{w}_m^{\mathrm{ele}})^\top
\mid {\scriptstyle \mathcal{W}}^o
\right]
=
w_m^{o,\mathrm{ele}}(w_m^{o,\mathrm{ele}})^\top+\mu\Phi_m.
\)
Step~(c) follows from Jensen's inequality.

In \eqref{equ:noise_terms}, the first component is a fluctuation term which measures
the deviation of the empirical second moment of the stochastic iterates
from its conditional mean. Under the block-diagonal covariance condition in
Theorem~\ref{Theorem: Laplacian estimation error}, the Gaussian
surrogate admits the representation
\begin{equation}
    \boldsymbol{w}_m^{\mathrm{ele}}
    =
    w_m^{o,\mathrm{ele}}
    +
    \sqrt{\mu}\Phi_m^{1/2}\boldsymbol{g}_m,
    \;where\; 
    \boldsymbol{g}_1,\ldots,\boldsymbol{g}_M
    \overset{\mathrm{i.i.d.}}{\sim}\mathcal{N}(0,I_K).
    \label{eq:decompose_distribution}
    \end{equation}
where $\boldsymbol{g}_m$ are
independent of ${\scriptstyle \mathcal{W}}^o$.

It admits the following bound after conditioning on the true model:
\begin{align}
    &\mathbb{E}_{\boldsymbol{w}_m^{\mathrm{ele}}|{\scriptstyle \mathcal{W}}^o}\!\left\|\!\frac{1}{M}\sum_{m=1}^{M}\left( \boldsymbol{w}_m^{\mathrm{ele}}(\boldsymbol{w}_m^{\mathrm{ele}})^\top\!\!-\mathbb{E}\left[\boldsymbol{w}_m^{\mathrm{ele}}(\boldsymbol{w}_m^{ele})^\top|{\scriptstyle \mathcal{W}}^o\right]\right)\!\right\|^2\notag\\
    &\overset{\eqref{eq:decompose_distribution}}{=} \mathbb{E}_{\boldsymbol{w}_m^{\mathrm{ele}}|{\scriptstyle \mathcal{W}}^o}\Biggl\|\frac{1}{M}\sum_{m=1}^{M}\mu\Phi_m^{\frac{1}{2}}(\boldsymbol{g}_m\boldsymbol{g}_m^\top-I_K)\Phi_m^{\frac{1}{2}} \notag\\&\;\;\;\;+\frac{1}{M}\sum_{m=1}^{M}(w_{m}^{o,\mathrm{ele}}\boldsymbol{g}_m^\top \sqrt{\mu}\Phi_m^{\frac{1}{2}}+\sqrt{\mu}\Phi_m^{\frac{1}{2}}\boldsymbol{g}_m(w_{m}^{o,\mathrm{ele}})^\top)\Biggr\|^2\notag\\
    & \overset{(a)}{\leq}  2\mathbb{E}_{\boldsymbol{w}_m^{\mathrm{ele}}|{\scriptstyle \mathcal{W}}^o}\left\|\frac{1}{M}\sum_{m=1}^{M}\mu\Phi_m^{\frac{1}{2}}(\boldsymbol{g}_m\boldsymbol{g}_m^\top-I_K)\Phi_m^{\frac{1}{2}}\right\|^2  \notag\\&+2\mu\mathbb{E}_{\boldsymbol{w}_m^{\mathrm{ele}}|{\scriptstyle \mathcal{W}}^o}\left\|\frac{1}{M}\sum_{m=1}^{M}(w_{m}^{o,\mathrm{ele}}\boldsymbol{g}_m^\top\Phi_m^{\frac{1}{2}}+\Phi_m^{\frac{1}{2}}\boldsymbol{g}_m(w_{m}^{o,\mathrm{ele}})^\top)\right\|^2\notag\\
    &\overset{(b)}{\leq} \tau\mu^2\|\Phi\|^2\left(\frac{K}{M}+\frac{K^2}{M^2}\right)\notag\\&+2\mu\mathbb{E}_{\boldsymbol{w}_m^{\mathrm{ele}}|{\scriptstyle \mathcal{W}}^o}\left\|\frac{1}{M}\sum_{m=1}^{M}(w_{m}^{o,\mathrm{ele}}\boldsymbol{g}_m^\top \Phi_m^{\frac{1}{2}}+\Phi_m^{\frac{1}{2}}\boldsymbol{g}_m(w_{m}^{o,\mathrm{ele}})^\top)\right\|^2\notag\\
    &\overset{(c)}{\leq} \tau\mu^2\|\Phi\|^2\left(\frac{K}{M}+\frac{K^2}{M^2}\right)\notag\\ 
    &\hspace{2em}+\frac{4\mu}{M}\sum_{m=1}^{M}\left\|w_{m}^{o,\mathrm{ele}}\right\|^2\mathbb{E}_{\boldsymbol{w}_m^{\mathrm{ele}}|{\scriptstyle \mathcal{W}}^o}\left\|\Phi_m^{\frac{1}{2}}\boldsymbol{g}_m\right\|^2\notag\\
    &\overset{(d)}{=} \tau\mu^2\|\Phi\|^2\left(\frac{K}{M}+\frac{K^2}{M^2}\right) + \frac{4\mu}{M}
    \sum_{m=1}^{M}
    \|w_m^{o,\mathrm{ele}}\|^2
    \operatorname{Tr}(\Phi_m),
\label{equ:bound_of_first_noise_term}
\end{align}
where (a) and (c) follow from Jensen's inequality. Step~(d) uses $\mathbb{E}\|\Phi_m^{1/2}\boldsymbol{g}_m\|^2
=\operatorname{Tr}(\Phi_m)$. Inequality (b) follows from standard sample-covariance concentration results~\cite{concentration,High-dimensional-probability}, for some positive constant $\tau>0$. The second term in
\eqref{equ:noise_terms} describes the finite-sample variation of the
GMRF parameters and must therefore be averaged over
$\boldsymbol{{\scriptstyle \mathcal{W}}}^o$. A similar concentration
argument gives, for some absolute constant $\rho>0$,
\begin{align}
    \mathbb{E}_{\boldsymbol{{\scriptstyle \mathcal{W}}}^o}&
    \left\|
    \frac{1}{M}\sum_{m=1}^{M}
    \left(
    \boldsymbol{w}_m^{o,\mathrm{ele}}
    (\boldsymbol{w}_m^{o,\mathrm{ele}})^\top
    -\Sigma
    \right)
    \right\|^2\notag\\
    &\leq
    \rho\|L^{-1}\|^2
    \left(
    \frac{K}{M}
    +
    \frac{K^2}{M^2}
    \right).
    \label{equ:bound_of_second_term}
    \end{align}
We also average the second term in
\eqref{equ:bound_of_first_noise_term} over
$\boldsymbol{{\scriptstyle \mathcal{W}}}^o$. Since $\Phi_m$ is
deterministic with respect to the GMRF realization, and using \(
\mathbb{E}_{\boldsymbol{{\scriptstyle \mathcal{W}}}^o}
\left\|
\boldsymbol{w}_m^{o,\mathrm{ele}}
\right\|^2
=
\operatorname{Tr}(L^{-1})
\), we obtain
\begin{equation}
    \mathbb{E}_{\boldsymbol{{\scriptstyle \mathcal{W}}}^o}
    \left[
        \frac{4\mu}{M}
        \sum_{m=1}^{M}
        \|\boldsymbol{w}_m^{o,\mathrm{ele}}\|^2
        \operatorname{Tr}(\Phi_m)
    \right]
    =
    4\mu\operatorname{Tr}(L^{-1})
    \operatorname{Tr}\left(\frac{\Phi}{M}\right),
    \label{equ:tower_bound_of_first_term}
\end{equation}

Finally, by the tower property, \(\mathbb{E}[\cdot]
=
\mathbb{E}_{\boldsymbol{{\scriptstyle \mathcal{W}}}^o}
\left[
\mathbb{E}_{\boldsymbol{\mathcal{W}}_i
\mid \mathcal{W}^o}[\cdot]
\right]\). Combining \eqref{equ:noise_terms},
\eqref{equ:bound_of_first_noise_term},
\eqref{equ:bound_of_second_term}, and
\eqref{equ:tower_bound_of_first_term} yields
    \begin{align}   \mathbb{E}&\left\|\widehat{\boldsymbol{\Sigma}}_{\mathcal{N}_k,\mathcal{N}_k}-\Sigma_{\mathcal{N}_k,\mathcal{N}_k}\right\|^2 \notag\\
    \leq & \underbrace{3\left(\tau\left\|\Phi\right\|^2+\rho\left\|L^{-1}\right\|^2\right)\left(\frac{K}{M}+\frac{K^2}{M^2}\right) }_{\text{covariance concentration}}\notag\\ 
    &+\, \underbrace{12\mu\operatorname{Tr}(L^{-1})\operatorname{Tr}\left(\frac{\Phi}{M}\right)+3\mu^2\left\|\operatorname{Tr_K}\left(\frac{\Phi}{M}\right)\right\|^2}_{\text{learning-induced variance}}\label{covariance error bound}\\
    \triangleq& \,\,\epsilon.
    \end{align}
    
\paragraph{Inverse covariance estimation error}
Then we can bound the local inverse covariance estimation error.
By a standard inverse perturbation inequality~\cite{stewart1990stochastic}, we have
\begin{align}
\mathbb{E}&
\left\|
\widehat{\boldsymbol{\Gamma}}_{k}-\Gamma_{k}
\right\|^{2}
\notag\\&\le
\mathbb{E}
\!\left[
\left\|\Gamma_{k}\right\|^{2}\,
\left\|
\widehat{\boldsymbol{\Sigma}}_{\mathcal{N}_k,\mathcal{N}_k}^{-1}
\right\|^{2}\,
\left\|
\widehat{\boldsymbol{\Sigma}}_{\mathcal{N}_k,\mathcal{N}_k}
-
\Sigma_{\mathcal{N}_k,\mathcal{N}_k}
\right\|^{2}
\right]\notag\\
&\overset{(a)}{\le}
\frac{\|\Gamma_{k}\|^{4}}{(1-\zeta)^{2}}\,
\epsilon\\
&\overset{(b)}{\le}
\frac{\|L_{\mathcal{N}_k,\mathcal{N}_k}\|^{4}}{(1-\zeta)^{2}}\,
\epsilon.
\label{eq:gamma_bound}
\end{align}
Here, (a) applies Theorem~2.5 in~\cite{stewart1990stochastic} under the condition
\begin{equation}
    \|\Gamma_{k}\|\,
\bigl\|
\widehat{\boldsymbol{\Sigma}}_{\mathcal{N}_k,\mathcal{N}_k}
-
\Sigma_{\mathcal{N}_k,\mathcal{N}_k}
\bigr\|
\le \zeta < 1.
\label{eq: condition}
\end{equation}
Motivated by the bound in~\eqref{equ:noise_terms}, we consider sufficiently large \(M\) and sufficiently small \(\mu\) such that condition~\eqref{eq: condition} holds. Step~(b) follows from the Schur complement identity in~\eqref{eq:schur_complement}, which implies
\(
\|\Gamma_k\|\leq \|L_{\mathcal{N}_k,\mathcal{N}_k}\|.
\)

By \eqref{eq:L estimate constraints}, the local Laplacian estimation error at node~$k$ satisfies
\begin{align}
\mathbb{E}
\left\|
\widehat{\boldsymbol{L}}^{\prime}_{k,\mathcal{N}_k}
-
L_{k,\mathcal{N}_k}
\right\|^{2}
&=
\mathbb{E}
\left\|
\operatorname{S}(\widehat{\boldsymbol{\Gamma}}_{k})
-
\operatorname{S}(\Gamma_{k})
\right\|^{2}
\nonumber\\
&\overset{(a)}{\le}
\mathbb{E}
\left\|
\widehat{\boldsymbol{\Gamma}}_{k}
-
\Gamma_{k}
\right\|^{2},
\label{eq:L_to_Gamma}
\end{align}
where (a) follows from the non-expansiveness of the subselection operator $\operatorname{S}(\cdot)$.

The global Laplacian estimation error can now be bounded as
\begin{align}
\mathbb{E}
\left\|\widehat{\boldsymbol{L}}-L\right\|^{2}
&\le
\mathbb{E}
\left\|\widehat{\boldsymbol{L}}-L\right\|_F^{2}
\notag\\
&\hspace{-0.7em}\overset{\eqref{eq:L estimate constraints},\eqref{eq:GGL_projection}}{=}
\mathbb{E}
\left\|\mathcal{P}_{\mathcal{C}_{\mathcal{G}}}\left(
\tfrac12\left(
\widehat{\boldsymbol{L}}^{\prime}
+
(\widehat{\boldsymbol{L}}^{\prime})^\top
\right)\right)
-
L
\right\|_F^{2}
\nonumber\\
&\overset{(a)}{\le}
\mathbb{E}
\left\|
\widehat{\boldsymbol{L}}^{\prime}
-
L
\right\|_F^{2}
\notag\\
&=
\sum_{k=1}^{K}
\mathbb{E}
\left\|
\widehat{\boldsymbol{L}}^{\prime}_{k,\mathcal{N}_k}
-
L^{\prime}_{k,\mathcal{N}_k}
\right\|^{2}
\nonumber\\
&\overset{\eqref{eq:gamma_bound},\eqref{eq:L_to_Gamma}}{\le}
\sum_{k=1}^{K}
\frac{\|L_{\mathcal{N}_k,\mathcal{N}_k}\|^{4}}{(1-\zeta)^{2}}\,
\epsilon.
\label{eq:full_L_error_bound}
\end{align}

In step~(a), we used the non-expansiveness of the
Euclidean projection onto a closed convex set, and the non-expansiveness of the symmetrization operator under the Frobenius norm,
i.e., for any matrix $X$,
\(
\left\|\tfrac12(X+X^\top)\right\|_F \le \left\|X\right\|_F
\),
together with the symmetry of the true Laplacian~$L$.

\section*{Proof of Lemma~\ref{lemma:MSD bound}}\label{proof: MSD}

\noindent In~\cite{Adaptation,stefan_reg2}, the network MSD performance is characterized under the assumption that the Laplacian is a symmetric positive semidefinite matrix with the all-ones vector in its nullspace, i.e., a combinatorial graph Laplacian. In the present setting, however, the Laplacian is symmetric positive definite. As a result, the argument in~\cite{Adaptation,stefan_reg2} does not apply directly, and we therefore provide a modified proof below.

From \eqref{eq:DGD_recursion} and \eqref{eq:multitask_cost}, the diffusion error recursion can be written as
\begin{align}
    {\scriptstyle \mathcal{W}}^o_{\widehat{L}} - \boldsymbol{{\scriptstyle \mathcal{W}}}_{\widehat{L},i} &= (I-\mu^\prime\eta\widehat{\mathcal{L}})\bigl({\scriptstyle \mathcal{W}}^o_{\widehat{L}}-\boldsymbol{{\scriptstyle \mathcal{W}}}_{\widehat{L},i-1}+\mu^\prime\nabla J(\boldsymbol{{\scriptstyle \mathcal{W}}}_{\widehat{L},i-1})\notag\\ &\;\;\;\;+\mu^\prime \boldsymbol{s}_i(\boldsymbol{{\scriptstyle \mathcal{W}}}_{\widehat{L},i-1})\bigr)+ \mu^\prime\eta\mathcal{L}{\scriptstyle \mathcal{W}}^o_{\widehat{L}}.
\end{align}
Applying the mean-value theorem to the gradient term~\cite{stefan_reg1} and using the first-order optimality condition at ${\scriptstyle \mathcal{W}}^o_{\widehat{L}}$, we obtain
\begin{align}
    \widetilde{\boldsymbol{{\scriptstyle \mathcal{W}}}}_{\widehat{L},i}
    &= (I-\mu^\prime\eta\mathcal{\widehat{L}})(I-\mu^\prime\boldsymbol{\mathcal{H}}_{i-1})\widetilde{\boldsymbol{{\scriptstyle \mathcal{W}}}} _{\widehat{L},i-1} \notag\\&\;\;\;\;+\mu^\prime(I-\mu^\prime\eta\mathcal{\widehat{L}})\boldsymbol{s}_i(\boldsymbol{{\scriptstyle \mathcal{W}}}_{\widehat{L},i-1}) + \mu^{\prime2}\eta^2\mathcal{\widehat{L}}^2{\scriptstyle \mathcal{W}}^o_{L}\label{eq:error recursion}
\end{align}
where $\boldsymbol{\mathcal{H}}_{i-1}\triangleq\operatorname{diag}\{\boldsymbol{H}_{1,i-1}, \cdots,\boldsymbol{H}_{K,i-1}\}$ with
\begin{equation}
    \boldsymbol{H}_{k,i-1} = \int_{0}^{1} \nabla^{2} J_k\!\left(w^o_{\widehat{L},k} - t\,(w^o_{\widehat{L},k} - \boldsymbol{w}_{k,i-1})\right)\, dt.
\end{equation}

To facilitate the analysis, we next introduce the long-term model from~\cite{stefan_reg2,Adaptation}, which replaces the time-varying Hessian matrix in \eqref{eq:error recursion} by its limiting value:
\begin{align}
  \widetilde{\boldsymbol{{\scriptstyle \mathcal{W}}}}^{\prime}_{\widehat{L},i}
    &= (I-\mu^\prime\eta\mathcal{\widehat{L}})(I-\mu^\prime\mathcal{H})\widetilde{\boldsymbol{{\scriptstyle \mathcal{W}}}}^{\prime}_{\widehat{L},i-1} \notag\\&\;\;\;\;+\mu^\prime(I-\mu^\prime\eta\mathcal{\widehat{L}})\boldsymbol{s}_i(\boldsymbol{{\scriptstyle \mathcal{W}}}_{\widehat{L},i-1}) + \mu^{\prime2}\eta^2\mathcal{\widehat{L}}^2{\scriptstyle \mathcal{W}}^o_{\widehat{L}}\label{eq:long-term recursion}
\end{align}
where $\mathcal{H}\triangleq\operatorname{diag}\{H_{1}, \cdots,H_{K}\}$ and $H_k\triangleq\nabla^2J_k\left(w^o_{\widehat{L},k}\right)$. Define
\begin{equation}
    \mathcal{B}\triangleq(I-\mu^\prime\eta\mathcal{\widehat{L}})(I-\mu^\prime\mathcal{H}),
\qquad
\mathcal{A}\triangleq(I-\mu^\prime\eta\mathcal{\widehat{L}}).
\end{equation}
Let $\boldsymbol{\mathcal{F}}_{i-1}$ denote the filtration generated by the random process $\{\boldsymbol{{\scriptstyle \mathcal{W}}}^\prime_{\widehat{L},j}\}$ for all $j \le i-1$. Squaring both sides of \eqref{eq:long-term recursion} and conditioning on the filtration $\boldsymbol{\mathcal{F}}_{i-1}$ gives
\begin{align}
\mathbb{E}\!\left[\left\|\widetilde{\boldsymbol{{\scriptstyle \mathcal{W}}}}^\prime_{\widehat{L},i}\right\|^2 \,\middle|\, \boldsymbol{\mathcal{F}}_{i-1}\right]
&=
\mathbb{E}\left\|\mathcal{B}\widetilde{\boldsymbol{{\scriptstyle \mathcal{W}}}}^\prime_{\widehat{L},i-1}\right\|^2 \notag\\
+\mu^{\prime 2}&\,
\mathbb{E}\!\left[
\left\|\mathcal{A}\boldsymbol{s}_i(\boldsymbol{{\scriptstyle \mathcal{W}}}_{\widehat{L},i-1})\right\|^2
\,\middle|\, \boldsymbol{\mathcal{F}}_{i-1}
\right]
+\operatorname{O}(\mu^{\prime2}).
\end{align}
The cross terms involving $\boldsymbol{s}_i(\boldsymbol{{\scriptstyle \mathcal{W}}}_{\widehat{L},i-1})$ vanish because the gradient noise is conditionally zero mean. The terms involving $\mu^{\prime 2}\eta^2 \mathcal{\widehat{L}}^2 {\scriptstyle \mathcal{W}}_L^o$ are collected into the higher-order remainder $\operatorname{O}(\mu^{\prime 2})$.

Taking expectation on both sides yields
\begin{align}
\mathbb{E}\left\|\widetilde{\boldsymbol{{\scriptstyle \mathcal{W}}}}^\prime_{\widehat{L},i}\right\|^2&
=
\mathbb{E}\left\|\mathcal{B}\widetilde{\boldsymbol{{\scriptstyle \mathcal{W}}}}^\prime_{\widehat{L},i-1}\right\|^2
+\mu^{\prime 2}\,
\mathbb{E}
\left\|\mathcal{A}\boldsymbol{s}_i(\boldsymbol{{\scriptstyle \mathcal{W}}}^\prime_{\widehat{L},i-1})\right\|^2\notag\\
&\quad
+\operatorname{O}(\mu^{\prime2}).\\
\leq
\left\|\mathcal{B}\right\|^2&\mathbb{E}\left\|\widetilde{\boldsymbol{{\scriptstyle \mathcal{W}}}}^\prime_{\widehat{L},i-1}\right\|^2
+\mu^{\prime 2}\,\mathbb{E}\left[
\operatorname{Tr}\left(\mathcal{A}^\top\mathcal{A}\mathcal{R}_{s,i}(\boldsymbol{{\scriptstyle \mathcal{W}}}_{\widehat{L},i-1})\right)\right]\notag\\
&\quad
+\operatorname{O}(\mu^{\prime2}).
\end{align}

We next verify that this recursion is stable. For stability, it is sufficient to ensure that
\[
\|\mathcal{B}\|=\|(I-\mu^\prime\eta\mathcal{\widehat{L}})(I-\mu^\prime\mathcal{H})\|<1.
\]
First, whenever $0\le\mu^\prime\eta\le\frac{2}{\lambda_{\text{max}}(\mathcal{\widehat{L}})}$, we have
\begin{equation}
    \left\|I-\mu^\prime\eta\mathcal{L}\right\|\leq1.
\end{equation}
Second, by the bounded-Hessian assumption,
\begin{equation}
    \gamma_{\operatorname{max}}\triangleq\|I-\mu^\prime \mathcal{H}\|=\max \{|1-\mu^\prime \nu|,|1-\mu^\prime\delta|\}.\label{eq: expansion}
\end{equation}
Thus, for a sufficiently small choice of $\mu^\prime$, we have $\gamma_{\operatorname{max}}\leq1$, which implies $0<\|\mathcal{B}\|<1$.

It follows that
\begin{align}
    \limsup_{i \to \infty}\mathbb{E}\|\widetilde{\boldsymbol{{\scriptstyle \mathcal{W}}}}^\prime_{\widehat{L},i}\|^2
    &\leq\frac{\mu^{\prime2}\operatorname{Tr}\left((I-\mu^\prime\eta\mathcal{\widehat{L}})^2\mathcal{R}_s\right)+\operatorname{O}\left(\mu^{\prime2}\right)}{1-\left\|(I-\mu^\prime\eta\mathcal{\widehat{L}})(I-\mu^\prime\mathcal{H})\right\|^2}\\
    &\overset{(a)}{\leq}\frac{\mu^{\prime2}\operatorname{Tr}\left((I-\mu^\prime\eta\mathcal{\widehat{L}})^2\mathcal{R}_s\right)+\operatorname{O}\left(\mu^{\prime2}\right)}{\min\{2\mu^\prime\nu-\mu^{\prime2}\nu^2,\notag\;\;2\mu^\prime\delta-\mu^{\prime2}\delta^2\}}\\
    &=\frac{\mu^{\prime}\operatorname{Tr}\left((I-\mu^\prime\eta\mathcal{\widehat{L}})^2\mathcal{R}_s\right)}{\min\{2\nu-\mu^{\prime}\nu^2,\notag\;\;2\delta-\mu^{\prime}\delta^2\}}+\operatorname{O}\left(\mu^{\prime}\right)
\end{align}
where step (a) uses the expansion in \eqref{eq: expansion}.

Finally, Lemma 1 in \cite{stefan_reg2} states that
\begin{equation}
    \limsup_{i \to \infty}\mathbb{E}\|\widetilde{\boldsymbol{{\scriptstyle \mathcal{W}}}}^\prime_{\widehat{L},i}\|^2 = \limsup_{i \to \infty}\mathbb{E}\|\widetilde{\boldsymbol{{\scriptstyle \mathcal{W}}}}_{\widehat{L},i}\|^2 + \operatorname{O}\left(\mu^{\prime\frac{3}{2}}\right).
\end{equation}
Therefore, for sufficiently small $\mu^{\prime}$, we conclude that
\begin{equation}
    \limsup_{i \to \infty}\mathbb{E}\|\widetilde{\boldsymbol{{\scriptstyle \mathcal{W}}}}_{\widehat{L},i}\|^2\leq\frac{\mu^{\prime}\operatorname{Tr}\left((I-\mu^\prime\eta\mathcal{\widehat{L}})^2\mathcal{R}_s\right)}{\min\{2\nu-\mu^{\prime}\nu^2,\notag\;\;2\delta-\mu^{\prime}\delta^2\}}+\operatorname{O}\left(\mu^{\prime}\right).
\end{equation}

\section*{Proof of Lemma~\ref{lemma:Laplacian-induced bias}}\label{proof:L-induced bias}
\noindent According to the bounded Hessian Assumption~\ref{assumption:Risk function}, the Hessian of
\(
\mathcal{J}_{\widehat{L}}({\scriptstyle \mathcal{W}})
\triangleq
\mathcal{J}({\scriptstyle \mathcal{W}})
+\frac{\eta}{2}\,{\scriptstyle \mathcal{W}}^\top \mathcal{\widehat{L}}\,{\scriptstyle \mathcal{W}}
\)
satisfies
\begin{align}
\nabla^2\mathcal{J}_{\widehat{L}}({\scriptstyle \mathcal{W}})
&= \nabla^2\mathcal{J}({\scriptstyle \mathcal{W}}) + \eta \widehat{\mathcal{L}}
\label{eq:hess_1}\\
&\geq \nu I_{MK} + \eta \widehat{\mathcal{L}}
\label{eq:hess_2}\\
&\geq \bigl(\nu + \eta \lambda_{\min}(\widehat{\mathcal{L}})\bigr) I_{MK}.
\label{eq:hess_3}
\end{align}
Let \(\lambda_{\widehat{L}} \triangleq \nu + \eta \lambda_{\min}(\widehat{\mathcal{L}})\).
Since $\mathcal{J}$ is assumed $\nu$-strongly convex and
$\widehat{L}\succ0$, the regularized cost
$\mathcal{J}_{\widehat{L}}$ is strongly convex. It therefore follows that
\begin{align}
\bigl\|\Delta{\scriptstyle \mathcal{W}}_{\widehat{L}}^o\bigr\|^2
&\le
\frac{1}{\lambda_{\widehat{L}}}
\big[\nabla \mathcal{J}_{\widehat{L}}({\scriptstyle \mathcal{W}}^o_{L})
-\nabla \mathcal{J}_{\widehat{L}}({\scriptstyle \mathcal{W}}^o_{\widehat{L}})\big]^\top
({\scriptstyle \mathcal{W}}^o_{L}-{\scriptstyle \mathcal{W}}^o_{\widehat{L}})
\notag\\
&\overset{(a)}{=}
\frac{1}{\lambda_{\widehat{L}}}
\big[\nabla \mathcal{J}({\scriptstyle \mathcal{W}}^o_{L})
+\eta\widehat{\mathcal{L}}\,{\scriptstyle \mathcal{W}}^o_{L}\big]^\top
({\scriptstyle \mathcal{W}}^o_{L}-{\scriptstyle \mathcal{W}}^o_{\widehat{L}})
\notag\\
&\overset{(b)}{=}
\frac{1}{\lambda_{\widehat{L}}}
\big[-\eta\mathcal{L}\,{\scriptstyle \mathcal{W}}^o_{L}
+\eta\widehat{\mathcal{L}}\,{\scriptstyle \mathcal{W}}^o_{L}\big]^\top
({\scriptstyle \mathcal{W}}^o_{L}-{\scriptstyle \mathcal{W}}^o_{\widehat{L}})
\notag\\
&\overset{(c)}{\le}
\frac{\eta}{\lambda_{\widehat{L}}}
\|\widehat{L}-L\|\,
\|{\scriptstyle \mathcal{W}}^o_{L}\|\,
\bigl\|\Delta{\scriptstyle \mathcal{W}}_{\widehat{L}}^o\bigr\|.
\label{eq:bias_81}
\end{align}
Here, (a) uses the optimality condition
\(\nabla \mathcal{J}_{\widehat{L}}({\scriptstyle \mathcal{W}}^o_{\widehat{L}})=0\);
(b) uses \(\nabla \mathcal{J}({\scriptstyle \mathcal{W}}^o_{L})+\eta\mathcal{L}{\scriptstyle \mathcal{W}}^o_{L}=0\);
and (c) follows from the Cauchy--Schwarz inequality.
Cancelling \(\bigl\|\Delta{\scriptstyle \mathcal{W}}_{\widehat{L}}^o\bigr\|\) on both sides of~\eqref{eq:bias_81} yields
\begin{equation}
\bigl\|\Delta{\scriptstyle \mathcal{W}}_{\widehat{L}}^o\bigr\|^{2}
\le
\frac{\eta^{2}}{\lambda_{\widehat{L}}^{2}}\,
\|\widehat{L}-L\|^{2}\,
\|{\scriptstyle \mathcal{W}}^o_{L}\|^{2}.
\end{equation}

\bibliographystyle{IEEEtran}
\bibliography{sample}













\end{document}